\documentclass[10pt]{article}

\usepackage[utf8]{inputenc}
\usepackage[T1]{fontenc}

\usepackage{epsf}
\usepackage{amsmath}

\allowdisplaybreaks

\usepackage[showframe=false]{geometry}
\usepackage{changepage}

\usepackage{epsfig}
\usepackage{amssymb}

\usepackage{amsthm}
\usepackage{setspace}
\usepackage{cite}
\usepackage{mcite}

\usepackage{algorithmic}  
\usepackage{algorithm}

\usepackage{shadow}
\usepackage{fancybox}
\usepackage{fancyhdr}

\usepackage{color}
\usepackage[usenames,dvipsnames,svgnames,table]{xcolor}
\newcommand{\bl}[1]{\textcolor{blue}{#1}}

\definecolor{mypurple}{rgb}{.4,.0,.5}

\usepackage[hyphens]{url}

\usepackage[colorlinks=true,
            linkcolor=black,
            urlcolor=blue,
            citecolor=purple]{hyperref}

\usepackage{breakurl}

\def\r{{\bf r}}
\def\y{{\bf y}}

\def\x{{\bf x}}

\def\x{{\mathbf x}}

\def\r{{\bf r}}

\def\x{{\bf x}}
\def\y{{\bf y}}
\def\z{{\bf z}}

\def\a{{\bf a}}

\def\d{{\bf d}}

\def\h{{\bf h}}

\def\cL{{\mathcal L}}

\def\cG{{\mathcal G}}

\def\be{\begin{equation}}
\def\ee{\end{equation}}
\def\ba{\left[\begin{array}}
\def\ea{\end{array}\right]}

\def\r{{\bf r}}

\def\x{{\bf x}}
\def\y{{\bf y}}
\def\z{{\bf z}}

\def\a{{\bf a}}

\def\d{{\bf d}}

\def\1{{\bf 1}}

\def\g{{\bf g}}
\def\0{{\bf 0}}

\def\erf{\mbox{erf}}
\def\erfc{\mbox{erfc}}

\def\mR{{\mathbb R}}

\def\mE{{\mathbb E}}

\def\mP{{\mathbb P}}

\def\lp{\left (}
\def\rp{\right )}

\def\r{{\bf r}}
\def\y{{\bf y}}

\def\x{{\bf x}}

\def\x{{\mathbf x}}

\def\r{{\bf r}}

\def\x{{\bf x}}
\def\y{{\bf y}}
\def\z{{\bf z}}

\def\a{{\bf a}}

\def\d{{\bf d}}

\def\h{{\bf h}}

\def\be{\begin{equation}}
\def\ee{\end{equation}}
\def\ba{\left[\begin{array}}
\def\ea{\end{array}\right]}

\def\r{{\bf r}}

\def\x{{\bf x}}
\def\y{{\bf y}}
\def\z{{\bf z}}

\def\a{{\bf a}}

\def\d{{\bf d}}

\def\R{{\bf R}}

\def\({\left (}
\def\){\right )}

\def\1{{\bf 1}}

\def\g{{\bf g}}
\def\0{{\bf 0}}

\usepackage{xcolor}
\usepackage{color}

\definecolor{darkgreen}{rgb}{0, 0.4,0}

\definecolor{purplebrown}{rgb}{0.5,0.1,0.6}

\definecolor{ultclupcol}{rgb}{0.1,0.5,0.5}

\definecolor{mytrycolor}{rgb}{0.5,0.7,0.2}

\definecolor{ultclupcola}{rgb}{.5,0,.5}

\definecolor{shadebrown}{rgb}{0.1,0.1,0.9}
\definecolor{lightblue}{rgb}{0.2,0,1}

\usepackage{fancybox}
\usepackage{graphicx}
\usepackage{epstopdf}
\usepackage{epsfig}
\usepackage{wrapfig}
\usepackage{subfigure}

\usepackage{xcolor}
\usepackage{tcolorbox}

\newtcbox{\xmybox}{on line,
arc=7pt,
before upper={\rule[-3pt]{0pt}{10pt}},boxrule=0pt,
boxsep=0pt,left=6pt,right=6pt,top=0pt,bottom=0pt,enhanced, coltext=blue, colback=white!10!yellow}

\newtcbox{\xmyboxa}{on line,
arc=7pt,
before upper={\rule[-3pt]{0pt}{10pt}},boxrule=0pt,
boxsep=0pt,left=6pt,right=6pt,top=0pt,bottom=0pt,enhanced, colback=white!10!yellow}

\newtcbox{\xmyboxb}{on line,
arc=7pt,
before upper={\rule[-3pt]{0pt}{10pt}},boxrule=1pt,colframe=darkgreen!100!blue,
boxsep=0pt,left=6pt,right=6pt,top=0pt,bottom=0pt,enhanced, colback=white!10!yellow}

\newtcbox{\xmyboxc}{on line,
arc=7pt,
before upper={\rule[-3pt]{0pt}{10pt}},boxrule=.7pt,colframe=blue!100!blue,
boxsep=0pt,left=6pt,right=6pt,top=0pt,bottom=0pt,enhanced, coltext=blue, colback=white!10!yellow}

\newtcbox{\xmytboxa}{on line,
arc=7pt,
before upper={\rule[-3pt]{0pt}{10pt}},boxrule=.0pt,colframe=pink!50!yellow,
boxsep=0pt,left=6pt,right=6pt,top=0pt,bottom=0pt,enhanced, coltext=white, colback=blue!40!red}

\newtcbox{\xmytboxb}{on line,
arc=7pt,
before upper={\rule[-3pt]{0pt}{10pt}},boxrule=.0pt,colframe=pink!50!yellow,
boxsep=0pt,left=6pt,right=6pt,top=0pt,bottom=0pt,enhanced, coltext=white, colback=white!40!green}

\makeatletter
\newcommand\subsubsubsection{\@startsection{paragraph}{4}{\z@}{-2.5ex\@plus -1ex \@minus -.25ex}{1.25ex \@plus .25ex}{\normalfont\normalsize\bfseries}}
\newcommand\subsubsubsubsection{\@startsection{subparagraph}{5}{\z@}{-2.5ex\@plus -1ex \@minus -.25ex}{1.25ex \@plus .25ex}{\normalfont\normalsize\bfseries}}
\makeatother

\newtheorem{theorem}{Theorem}

\newtheorem{lemma}{Lemma}

\begin{document}

\begin{singlespace}

\title{Optimal spectral initializers impact on phase retrieval phase transitions -- an RDT view 
}
\author{
\textsc{Mihailo Stojnic
\footnote{e-mail: {\tt flatoyer@gmail.com}} }}
\date{}
\maketitle

\centerline{{\bf Abstract}} \vspace*{0.1in}

We analyze the relation between spectral initializers and theoretical limits of \emph{descending} phase retrieval algorithms (dPR). In companion paper \cite{Stojnicphretreal24}, for any sample complexity ratio, $\alpha$, \emph{parametric manifold}, ${\mathcal {PM}}(\alpha)$, is recognized as a critically important structure that generically determines dPRs abilities to solve phase retrieval (PR). Moreover, overlap between the algorithmic solution and the true signal is positioned as a key ${\mathcal {PM}}$'s component. We here consider the so-called \emph{overlap optimal} spectral initializers (OptSpins) as dPR's starting points and  develop a generic \emph{Random duality theory} (RDT) based program to statistically characterize them. In particular, we determine the functional structure of OptSpins and evaluate the starting overlaps that they provide for the dPRs. Since ${\mathcal {PM}}$'s so-called \emph{flat regions} are highly susceptible to \emph{local jitteriness} and as such are key obstacles on dPR's path towards PR's global optimum, a precise characterization of the starting overlap allows to determine if such regions can be successfully circumvented. Through the presented theoretical analysis we observe two key points in that regard: \textbf{\emph{(i)}} dPR's  theoretical phase transition (critical $\alpha$ above which they solve PR) might be difficult to practically achieve as the ${\mathcal {PM}}$'s flat regions are large causing the associated OptSpins to fall exactly within them; and \textbf{\emph{(ii)}} Opting for so-called ``\emph{safer compression}'' and slightly increasing  $\alpha$ (by say $15\%$) shrinks flat regions and allows OptSpins to fall outside them and dPRs to ultimately solve PR. Numerical simulations are conducted as well and shown to be in an excellent agreement with theoretical predictions.

\vspace*{0.25in} \noindent {\bf Index Terms: Phase retrieval; Descending algorithms; Spectral initializers; Random duality theory}.

\end{singlespace}

\section{Introduction}
\label{sec:back}

Let $\bar{\x}\in\R^n$ be a unit norm data vector and let $\bar{\y}$ be the following collection of its \emph{phaseless} measurements
\begin{eqnarray}
   \bar{\y} &=& |A\bar{\x}|^2. \label{eq:inteq1}
\end{eqnarray}
Assuming that one has access to  $\bar{\y}$ and the so-called measurement matrix  $A\in\mR^{m\times n}$, recovering $\bar{\x}$ boils down to solving the following inverse feasibility problem
\begin{eqnarray}
\mbox{find} & & \x \nonumber \\
   \mbox{subject to} & & |A\x|^2=\bar{\y} (= |A\bar{\x}|^2). \label{eq:inteq2}
\end{eqnarray}
The system of quadratic equations in (\ref{eq:inteq2}) is the real variant of the famous \emph{phase retrieval} (PR) problem. PRs are the key mathematical concepts behind a plethora of signal and image processing data acquisitions and recovery protocols where accessing signal's phase is unavailable through conventional measuring techniques. History of these applications is rather reach and dates back to the first half of the last century and early days of x-ray crystallography \cite{Harrison93,Millane90,Millane06}. Following  x-ray crystallography expansion to noncrystalline materials, PR became a key mathematical component of coherent diffraction, ptychography, astronomical, optical, or microscopic imaging (see, e.g., \cite{Thibault08,Hurt89,KST1995,Miao1999ExtendingTM,ShechtmanECCMS15,MISE08,Bunk07,Walther01011963,Fienup87,Fienup78,Rod08,Dierolf10,BS79,Misell73}).  Further extensions including both imaging and non-imaging related ones followed as well. Examples include digital holography \cite{Duadi11,Gabor48,Gabor65}, blind deconvolution/demixing \cite{MWCC19,LLSW19,Jung17,ARJ13}, quantum physics \cite{Corb06,Hein13,HaahHJWY17}, and many others.

In this paper we focus on PR's mathematical aspects. Two groups of problems are of predominant interest on that front: \textbf{\emph{(i)}} The first group encompasses problems related to ``mathematical soundness'' of (\ref{eq:inteq2}) whereas \textbf{\emph{(ii)}} the second group puts the emphasis on actual algorithmic solving of (\ref{eq:inteq2}). Within the first group one typically looks at algebraic  characterizations of uniqueness/injectivity and stability  properties that stem from the PR's inherent phase ambiguity (see, e.g., \cite{Conca15,Balan06,Balan09,Bande14,Vinz15} and references therein). Roughly speaking, the most typical results state that (\ref{eq:inteq2}) is solvable in the so-called ``\emph{strong sense}''  provided that $\alpha\geq 2$ in the real and $\alpha\geq 4$ in the complex case \cite{Conca15,Balan06,Balan09,Bande14,Vinz15}. If one relaxes a bit the strong notion (which assumes solvability for \emph{any} $\bar{\x}$) and instead considers more practical ``\emph{weak sense}'' solvability (for say \emph{a} given $\bar{\x}$) then $\alpha\geq 1$ in the real and $\alpha\geq 2$ in the complex case suffices \cite{MaillardKLZ21}.

As this paper is closer to the second group, we take a more detailed look at typical solving strategies. Simply stated, solving (\ref{eq:inteq2}) assumes actual recovery of $\bar{\x}$. Since both $\bar{\x}$ and $-\bar{\x}$ are admissible solutions, (\ref{eq:inteq2}) can be solved only up to a \emph{global} phase (in the real case that means up to the sign of $\bar{\x}$). Consequently, throughout the paper we under solving (\ref{eq:inteq2}) always assume solving it up to a \emph{global} phase (also, we only  consider non-degenerate \emph{solvable} scenarios, where $A$ is such that besides $\pm \bar{\x}$  no other solution exists). Keeping all these technicalities in mind, one can then say that (\ref{eq:inteq2}) is conceptually readily solvable. For example, naively going over all $2^m$ different  sign options for the components of $\bar{\y}$ and solving the residual linear problems suffices. This is, however, computationally prohibitive when $n$ and $m$ are large which is the regime  of interest in a majority of practical applications. As it is also of interest in this paper, we mathematically formalize it as \emph{linear} (often referred to as proportional) high-dimensional regime characterized by the \emph{oversampling or sample complexity ratio}
\begin{eqnarray}
 \alpha \triangleq \lim_{n\rightarrow\infty} \frac{m}{n}, \label{eq:inteq3}
\end{eqnarray}
which remains constant as the size of data, $n$, and the sample complexity, $m$, grow. Within such a context the above mentioned naive exhaustive search is of \emph{exponential} complexity and practically unacceptable (particularly so if one has in mind that imaging PR applications often assume sizes of data vectors $n\sim 10000$).

Since it was clear from the early PR days that naive concepts can not satisfy practical demands, finding computationally efficient alternatives positioned itself as a key mathematical challenge. A systematic development of a vast algorithmic theory naturally ensued resulting in a steady progress through the better portion of the second half of the last century \cite{Gerch72,Gabor65,Misell73,Fienup82,Fienup78,Fienup87}. Perhaps somewhat unexpectedly, research particularly intensified over the lat two decades. Several extraordinary algorithmic and theoretical breakthroughs  played a key role in such a rapid PR revival, see, e.g., \cite{CandesSV13,CandesESV13,CandesLS15}. One of these breakthroughs uncovered that simple gradient descent based algorithmic techniques  (conveniently called Wirtinger flow in \cite{CandesLS15})  surprisingly perform very well despite a lack of (presumably necessary) underlying convexity. In parallel, \cite{CandesLS15} also recognized the role played by the so-called initializers -- the starting points of the gradient algorithms. Companion paper \cite{Stojnicphretreal24} studied generic \emph{descending} phase retrieval algorithms (dPR) (which include as special cases any forms of gradient descent or Wirtinger flow) and provided a \emph{precise} phase transition type of theoretical analysis of their performance. Here we complement such an analysis with the corresponding one related to the initializers and connect the two ultimately completing the overall  mosaic regarding theoretical justification of a rather remarkable practical success of these classical optimization techniques. Before proceeding with the presentation of our results we briefly survey related technical prior work.

\subsection{Relevant prior work and key analytical features}
\label{sec:relwork}

\subsubsection{Algorithmic methods}
\label{sec:relworkalg}

Many known algorithmic techniques have been adapted so that they can be utilized within the PR context and many new ones catering to specific PR's particularities have been developed as well. Quite a few also come with strong associated theoretical guarantees thereby supplying an additional justification for their use. We below revisit two separate groups of these techniques, convex and non-convex (in addition to convexity, various other groupings are possible as well; for example, they can be based on computational complexity, speed of convergence, easiness of practical implementation and so on).

\noindent $\star$ \underline{\emph{Convex methods}:}   PR in (\ref{eq:inteq2}) is a non-convex feasibility problem. Given that the constraint set can easily be recast as $\bar{\y}=\mbox{diag}(A^TXA), X=X^T\geq 0, \mbox{rank}(X)=1$, it is not that difficult to see that the non-convexity stems from the rank-1 matrix  constraint. Following standard semi-definite programming (SDP) methodology,  \cite{CandesSV13,CandesESV13} proposed the so-called Phaselift ``dropping rank-1'' relaxation  (for a closely connected Phasecut relaxation --  a direct maxcut PR analogue -- see \cite{WaldspurgerdM15}, and for related SDP considerations see also \cite{Ohlsen12}). \cite{CandesSV13,CandesESV13} went further and showed that Phaselift  both exactly and stably solves PR provided that $m=O(n\log(n))$ (a further Phaselift robustness discussion can be found in \cite{PHand17}).  In \cite{CandesL14}, the sample complexity estimate was lowered to $m=O(n)$, which is the optimal order. Analogous results when Fourier (instead of Gaussian) measurements are used are presented  in \cite{GFK17,CandesLS15b}. While the SDP  relaxations are typically among the tightest convex ones, they require solving the residual SDP which usually poses a serious numerical challenge in large dimensional contexts. \cite{GoldsteinS18,BahmaniR16} introduced a numerically more acceptable alternative, PhaseMax, and showed that it solves PR provided that the initializer is sufficiently close to $\bar{\x}$  (see also \cite{HandV16,ChenCandes17} for initializers that bring the sample complexity to $m=O(n)$ and \cite{SalehiAH18} for a precise PhaseMax analysis). A multi-iterative PhaseMax alternative, PhaseLamp, was considered in \cite{DhifallahTL17} and shown to have better theoretical and algorithmic properties than PhaseMax. Additional structuring of $\bar{\x}$ (sparsity, positivity, low-rankness and so on) can easily be incorporated in convex methods by following the practice typically adopted in standard compressed sensing (see, e.g. \cite{Ohlsen12,LiVor13,KeungRTi17}). For example, \cite{LiVor13} shows that $m=O(k^2\log(n))$ measurements suffice for successful PR recovery of a $k$-sparse $\bar{\x}$ which is even for $m\ll n$ (a nonlinear regime not of our prevalent interest) surprisingly weaker than the corresponding  compressed sensing result $m=O(k\log(n))$. Many upgrades of regular convex techniques are possible as well when $\bar{\x}$ is additionally structured (see, e.g.,  randomized Kaczmarz adaptations in \cite{TV19,KWei15} and  two-stage approaches in \cite{JaganathanOH17,IwenVW17}).

\noindent $\star$ \underline{\emph{Non-convex methods}:} While the modern optimization theory practice commonly suggests first trying convex and then non-convex methods, within the PR context, things were reversed, i.e., the utilization  of the non-convex  methods significantly predated the convex ones. The early PR algorithms typically related to the so-called alternating minimization \cite{Gerch72,Fienup82,Fienup87}. As stated above, a big breakthrough arrived when \cite{CandesLS15} uncovered that the simple gradient (Wirtinger flow) often significantly outperforms all known techniques (for a closely related variant called amplitude flow, see, e.g., \cite{WangGE18} and  for further modifications, e.g.,  \cite{CLM16}). Moreover, not only was the (non-convex) gradient empirically outperforming the (convex) SDPs, but the accompanying theoretical analysis of  \cite{CandesLS15} revealed that it is doing so while allowing for $m=O(n\log(n))$ sample complexity (sparse adaptations of Wirtinger and
amplitude flow \cite{YuanWW19,WangGE18}, follow the suit of the convex methods and allow for $m=O(k^2\log(n))$; on the other hand, \cite{Soltanolkotabi19} in a way breaks the unpleasant $k^2$ barrier and shows that a local Wirtinger flow convergence can be achieved with $m=O(k\log(n))$). Prior to the appearance of \cite{Stojnicphretreal24}, this was viewed as fairly surprising given that the preliminary considerations of \cite{SunQW18,HandLV18}  indicated PR objective's very strong deviations from convexity. Another non-convex
classical method, Alternating direction method of multipliers (ADMM) (practically Douglas-Rachford in the PR context), was considered in \cite{FannjiangZ20} and shown to perform very well (for further theoretical characterizations of alternating minimization projections, see also \cite{Netrapalli0S15,Waldspurger18,Yanget94,March07,March14}).

Drawing parallel with the promise generative models and deep learning have shown in compressed sensing
\cite{LeiJDD19,JordanD20,BoraJPD17,DaskalakisRZ20a}, \cite{HandLV18} introduced a Deep (sparse) phase retrieval (DPR) concept. Empirical results on par with those from \cite{BoraJPD17}  were obtained, indicating deep nets' strong potential to substantially improve on non-machine learning based techniques. \cite{HandLV18} provided additional theoretical support as well and proved that $d$-layer nets (with a constant expansion) can achieve such a performance with $m=O(kd\log(n))$. For a constant $d$, this allows breaking the above $k^2\log(n)$ barrier and achieving the $k\log(n)$ sample complexity, precisely as predicated through  the analogy with standard compressed sensing. On the other hand, benefits of deep learning (DL) come hand in hand with some limitations as well. Nonzero errors are hard or impossible to achieve which basically renders DL as noncompetitive with exact methods  in the so-called realizable or noiseless scenarios. Also, allowed problem dimensions are often restrictive and frequent retraining might be time-consuming. Despite a strong progress recently made in \cite{MBBDN23,Stojnicinjdeeprelu24,Stojnicinjrelu24}, precise theoretical characterizations are significantly harder to obtain and PR utilization of deep nets is effectively still a bit distant from expected generic superiority. A generically more competitive class of methods includes heuristic adaptations of famous approximate message passing (AMP) algorithm (a large scale compressed sensing alternative introduced in  \cite{DonMalMon09}). While such methods and the accompanying theoretical analyses often heavily depend on model assumptions, an excellent empirical performance within phase retrieval context is shown in \cite{SchniterR15}.

\subsubsection{Initializers}
\label{sec:relworkinit}

\noindent $\star$ \underline{\emph{Role of initializers}:} Both convex and non-convex groups of algorithms require a starting point -- \emph{initializer}. As initializers do not affect the accuracy of convex methods their  role within this group is less pronounced and typically related to the speed of convergence. Moreover, since convex methods are generally viewed as implementable in polynomial time the initializers' impact on computational complexity is also rather limited. On the other hand, when it comes to the non-convex alternatives, things are way different. The initializers can critically impact both their ability to reach the global optima and their computational efficiency. As stated earlier, figuring out the intricacies of the initializers' roles within non-convex methods in a precise manner is the topic of this paper. There have been some excellent related prior results in these directions as well and we briefly look at them below.

\noindent $\star$ \underline{\emph{Prior initializers results}:} Within phase retrieval context, spectral initializers appeared as an integral component alternating minimization approach of \cite{Netrapalli0S15} (see also \cite{Netrapalli0S13} for related and \cite{Li92} for earlier similar considerations). Within the context of gradient methods, \cite{CandesLS15} proposed simple diagonal spectral initializers to start Wirtinger flow.  \cite{CandesLS15,Netrapalli0S15,Netrapalli0S13} also showed that not only can the spectral initializers serve as good starters for non-convex methods but they can themselves solve PR provided that $m$ is large enough.  Using a truncated variant of diagonal initializers, \cite{ChenCandes17} showed that this can be achieved even in a linear regime, i.e., for sample complexity $m=O(n)$. However, as the associated proportionality constants are rather large in practice spectral methods are commonly perceived as good initializers rather than direct PR solvers. To further substantiate such a common intuition,  precise performance characterizations of spectral initializers naturally appeared as necessary. \cite{LuL17} built on the above initial success of spectral methods  and considered a more general nonnegative diagonal spectral variant (see also \cite{LuLi20} for further extensions) and obtained a precise ``overlap vs oversampling ratio'' characterization. It further uncovered that the oversampling ratio undergoes a sharp transition between phases where the overlap  is zero or nonzero. This immediately had direct consequences on non-convex PR algorithms as their prowess is strongly predicated on the use of initializers with nonzero overlap. Results of \cite{LuL17} were then extended  to incorporate negative diagonals (and complex domain) in \cite{MondelliM19}. \cite{MondelliM19} also introduced the concept of ``weak threshold'' as critical oversampling ratio below which no diagonal spectral initializer can achieve nonzero overlap. Moreover, \cite{MondelliM19} precisely determined the value of such a threshold. In \cite{LuoAL19} the optimality of spectral  preprocessing analyzed in \cite{MondelliM19} was proven for any sample complexity (i.e., not only for the critical one corresponding to the weak threshold). Complementing results of  \cite{LuL17,MondelliM19} that relied on Gaussian measurements, \cite{MaDXMW21} considered practically more attractive orthogonal ones and utilized an Expectation Propagation paradigm to precisely characterize associated ``overlap vs oversampling ratio'' (even more general measurement models were considered via statistical physics tools in \cite{AubinLBKZ20,MaillardKLZ21}). A mathematically fully rigorous confirmation of the results from  \cite{MaDXMW21} was presented in \cite{DudejaB0M20}.

\subsubsection{Types of performance analysis -- \emph{``Qualitative'' vs ``quantitative''}}
\label{sec:relworktypesanal}

\noindent $\star$ \underline{\emph{Relevance of quantitative analysis}:} Keeping all parameters equal, the most desirable feature of any PR algorithm is the lowest allowed sample complexity.  Among the convex methods, the SDP relaxations (Phaselift and its derivatives) have the lowest sample complexity whereas their linear alternatives (PhaseMax and its derivatives) are computationally the most efficient. A natural question is if the algorithms that achieve both low sample complexity and high computational efficiency can be designed. Empirical results suggest that some of the non-convex methods discussed above (including gradient or Wirtinger flow) to a large degree do achieve that. Theoretically analyzing them while relying on the usual ``\emph{qualitative}'' types of analysis is unlikely to provide convincing arguments that would allow to capture needed comparative subtleties.  In order to bridge the gap, more precise \emph{quantitative} approaches are required. For example, when one practically runs spectral methods in real case they require $\alpha\sim 10$ whereas the corresponding running of gradient of Wirtinger flow typically requires $\alpha\sim 2.1$. Clearly the latter is way more preferable in practice. As qualitative theoretical analyses only show the linearity of $m$ for both methods, they do not suffice to convincingly distinguish which linearity is practically better and by how much.

\noindent $\star$ \underline{\emph{``Quantitative vs qualitative''}:} The above is just a tiny example. Way more general principle holds. Existing theoretical analyses associated with the \emph{majority} of the non-convex methods discussed above are actually of the qualitative type. As such they typically provide correct orders of magnitudes of the analyzed quantities which is usually sufficient to justify algorithms' behavior and conduct some basic/rough comparisons with competitive alternatives. However, when more accurate comparisons between similar alternatives are needed the level of precision of the qualitative analyses does not suffice which renders obtaining precise, \emph{quantitative}, characterizations as critically important. However, these are technically a lot harder and the known results are consequently much scarcer. For example, a replica analysis of the vector AMP (VAMP) algorithms \cite{SchniterRF16,RanganSF17} is presented in \cite{TakahashiK22}. Bayesian inference context similar to the one from  \cite{TakahashiK22}  is considered in \cite{MaillardLKZ20} and a large set of replica predictions confirmed (a real version of the complex counterpart from \cite{MaillardLKZ20} can be found in \cite{BarbierKMMZ18}; another excellent example of  Bayesian context utilization can be found in  \cite{StrSag25}, where a two-stage phase-selection  procedure is proposed and analyzed via replica methods). As Bayesian context in general heavily relies on assuming statistical identicalness of $\bar{\x}$'s components and a perfect prior knowledge of statistics of both $\bar{\x}$ and the so-called channel (or posterior), it is not easily amenable for a fair comparison with the methods discussed above that do not require such assumptions. On a different token though, it does provide precise, phase transition type of results that are far away from being reachable by any of the simpler qualitative methods discussed above. One should also keep in mind that the qualitative methods are by no means easy to come up with. They are often hard on their own and the quantitative ones are just much harder.

\subsection{Our contributions}
\label{sec:contrib}

Our focus is on \emph{precise} (quantitative) theoretical performance characterizations of non-convex phase retrieval optimization algorithms. In companion paper \cite{Stojnicphretreal24}, the so-called \emph{descending} phase retrieval algorithms (dPR) are considered. Here we focus on the relation between the overlap optimal spectral initializers (OptSpins) and dPRs. We recall that we consider mathematically most challenging \emph{linear} high-dimensional regime  with sample complexity ratio $\alpha=\lim_{n\rightarrow \infty} \frac{m}{n}$ remaining constant as the length of the unknown vector, $n$, and the number of phaseless measurements, $m$, grow. Similarly to  \cite{Stojnicphretreal24}, for the easiness of exposition we assume standard normal measurements and consider the real valued scenario. Since additional technical adjustments are needed in the complex case, to avoid sidetracking main discussions, we defer presenting them to a separate companion paper (however, we do emphasize that everything presented here remains conceptually unaltered in the complex case).

\begin{itemize}
\item As \cite{Stojnicphretreal24} showed, the shape of the phase retrieval associated  \emph{parametric manifold}, ${\mathcal {PM}}(\alpha)$ critically impacts dPRs abilities to solve phase retrieval. The overlap between the phase retrieval optimizing variable, $\x$, and its true (desired) solution, $\bar{\x}$, is positioned as a key ${\mathcal {PM}}$'s component. Intuitively, larger the overlap of the dPR's starting point more favorable its position within the ${\mathcal {PM}}$ and more likely the ultimate dPR's success. Following such an intuition, we distinguish between \emph{general} and \emph{overlap optimal} spectral initializers (OptSpins) (see Section \ref{sec:specinit}) and observe that precisely characterizing the performance of the latter directly impacts practical and theoretical dPR's phase transitions (see Section \ref{sec:algimp}).

\item We formulate a convenient \emph{fundamental  spectral initializers optimization} (f-spin) and recognize that studying behavior of its objective is of key importance for understanding the performance of OptSpins (see part 1) in Section \ref{sec:ubrdt}).

  \item To study f-spin we develop a \emph{Random duality theory}  (RDT) based generic statistical  program  (see parts 2) and 3) in Section \ref{sec:ubrdt}).

  \item  The developed RDT program allows to \emph{precisely} determine the starting dPR overlap values  if OptSpins are used as a preprocessing step before running any dPR (see Figure \ref{fig:fig1} in Section \ref{sec:ubrdt}). The overlap values precisely match the corresponding ones obtained through random matrix and free probability spectral theory in  \cite{MondelliM19,LuoAL19}. However, the RDT methods that we present are completely different, simpler, and  more general (for example, they easily allow for further extensions and incorporation of additionally structured signals including sparse, block-sparse, positive, binary, box-constrained, partially observed, and many others).

\item Once starting dPR overlap values are determined one can superimpose them on ${\mathcal {PM}}(\alpha)$ and check what portion/region of the manifold they fall in. Depending on $\alpha$,  ${\mathcal {PM}}(\alpha)$ can exhibit large so-called \emph{flat regions} which are easily susceptible to local jitteriness finite dimensional effects and as such might pose an unsurpassable obstacle for dPRs. It is therefore recognized as highly desirable that starting dPR points fall outside such regions (see discussion after Theorem \ref{thm:thm1a} in Section \ref{sec:algimp}).

\item We demonstrate these concepts by first choosing two different sample complexity ratios $\alpha=1.4$ and $\alpha=1.6$, then determining the corresponding dPR starting overlap values from Figure \ref{fig:fig1}, and ultimately superimposing them on critical ${\mathcal {PM}}(\alpha)$'s regions in Figures \ref{fig:fig2} and \ref{fig:fig3}. The first $\alpha$ value is chosen as the dPR theoretical phase transition obtained in  \cite{Stojnicphretreal24} (this is the limiting $\alpha$ for which the lifted RDT ensures single funneling point structure of ${\mathcal {PM}}$ which in return directly corresponds to the dPR's ability to solve phase retrieval). On the other hand, the second $\alpha$ value is chosen as an example of the so-called ``\emph{safer compression}'' which allows for a bit of a buffer zone and supposedly a bit more comfortable (or so to say, away from the edge) running of dPR.

 \item   We observe two key points: \textbf{\emph{(i)}} dPR's global convergence theoretical phase transitions (predicated by the lifted RDT to be at $\alpha\approx 1.4$) might be practically difficult to achieve as ${\mathcal {PM}}$'s flat regions occupy large portions of the allowed $[0,1]$ overlap interval and the dPR's starting overlap values (provided by the associated OptSpins) fall exactly within such regions (i.e., for $\alpha\approx 1.4$, the OptSpins are unlikely to be able to allow dPR to circumvent ${\mathcal {PM}}$'s flat regions); and \textbf{\emph{(ii)}} if instead one opts for the ``safer compression'' and slightly increases $\alpha$ (by say $15\%$), the OptSpins fall outside flat regions and the dPRs are likely to solve phase retrieval.

 \item Numerical experiments are conducted as well and obtained results are shown to be in an excellent agreement with the theoretical predictions. In particular, we implement a hybrid combination of a log barrier and a plain gradient descent and observe that, even in small dimensional scenarios with  $n=300$ (which are highly likely to be strongly susceptible to jitteriness effects), the simulated phase transitions are fairly close to the above mentioned safer compression adjusted theoretical predictions (see Figure \ref{fig:fig7}).

 \item Following the trend of  \cite{Stojnicphretreal24}, we found utilization of the so-called \emph{squared} magnitudes objectives in practical algorithmic implementations as particularly convenient. On the other hand, the theoretical considerations (presented in Theorems \ref{thm:thm1a} and \ref{thm:thm2a} and visualized in Figures \ref{fig:fig2} and \ref{fig:fig3}),  were easier to conduct with the \emph{non-squared} magnitudes objectives.  Despite substantial numerical difficulties, we complement the theoretical non-squared results with the corresponding squared ones in Section \ref{sec:sqadj}. While the numerical precision of the residual evaluations is  naturally not as high, all main theoretical conclusions from the non-squared  translate to the corresponding squared magnitudes scenario (see Figures \ref{fig:fig2} and \ref{fig:fig4} and observe that the superimposed spectral initializers in both figures fall into the ``undesired'' flat regions).

\item We also note that the OptSpins can be utilized to lower the theoretical dPR phase transitions as well. A state of the art conceptual discussion in this direction is provided in Section \ref{sec:thimpact}.

\end{itemize}

\section{Phase retrieval -- basics}
 \label{sec:2lay}

As mentioned earlier, recovering unknown signal within phase retrieval context algorithmically boils down to finding unique (modulo global phase) solution of
 \begin{eqnarray}
 {\mathcal R}(A): \qquad \qquad    \min_{\x,\z} & & \||A\bar{\x}|-|\z| \|_2^2\nonumber \\
  \mbox{subject to} & &  A\x=\z. \label{eq:ex1a4}
\end{eqnarray}
Following into the footsteps of \cite{Stojnicinjrelu24,StojnicGardGen13,StojnicICASSP10var,StojnicCSetam09} and keeping in mind that measurements are random, the above \emph{random optimization problem} (rop) is in \cite{Stojnicphretreal24} conveniently associated with
 \begin{eqnarray}
\hspace{-.4in}\bl{\textbf{\emph{Fundamental phase-retrieval optimization (f-pro):}}} \qquad\qquad  \xi(c,x) \triangleq \min_{\x,\z} & & \||A\bar{\x}|-|\z| \|_2^2\nonumber \\
  \mbox{subject to} & &  A\x=\z \nonumber \\
  & & \x^T\bar{\x}=x \nonumber \\
  & & \|\x\|_2^2=c. \label{eq:ex1a4a0}
\end{eqnarray}
As discussed in above (and in detail in \cite{Stojnicphretreal24}), the formulations in  (\ref{eq:ex1a4}) and (\ref{eq:ex1a4a0}) use the so-called \emph{non-squared} magnitudes and correspond to what is in literature typically called \emph{amplitude} measurements. A slightly more prevalent formulation would instead make use of the \emph{squared} magnitudes and would consequently correspond to the so-called \emph{intensity} measurements. In a further alignment with the above discussion, \cite{Stojnicphretreal24} also showed that the non-squared formulations (precisely those stated in (\ref{eq:ex1a4}) and (\ref{eq:ex1a4a0}))  are more convenient from the analytical point of view whereas the squared ones are easier for practical implementations. Since \cite{Stojnicphretreal24} also showed that not much of a conceptual difference between the two appears to exist, the advantage here is given to analytically more elegant and numerically less cumbersome non-squared ones (for a brief discussion about the squared ones see Section \ref{sec:sqadj}).

Looking at (\ref{eq:ex1a4a0}), it is not that difficult to guess that two parameters attributed to the optimizing variable (its squared norm, $c$, and its overlap with the true solution, $x$) play a key role in characterizing the accuracy of any associated optimization algorithm. In fact, \cite{Stojnicphretreal24} provided a precise analysis of any \emph{descending} algorithm and identified structure (shape) of a parametric manifold of pairs $(c,x)$ as of critical importance for ensuring that such algorithms ultimately succeed in solving phase retrieval. In addition to such a theoretical reassurance, a long standing common belief regarding practical (algorithmic) aspects of  (\ref{eq:ex1a4}) is the following: vectors that have larger overlap with the true $\bar{\x}$ (i.e., vectors that result in larger $x$ in (\ref{eq:ex1a4a0})) are likely to serve as better algorithm's starting points (\emph{initializers}). Practical implementations (including the most common ones based on convex matrix rank relaxations \cite{CandesL14,CandesSV13}, Wirtinger flow \cite{CandesLS15}, or approximate message passing (AMP) \cite{SchniterR15,MaillardLKZ20,MaillardKLZ21,AubinLBKZ20}) seem to support such an intuition as well. On the other hand, theoretical analyses that would precisely confirm this are much harder to obtain. A strong progress in that direction has been made in companion paper \cite{Stojnicphretreal24} where the overlap is directly connected to the phase transitional (PT) behavior of \emph{all descending} phase retrieval algorithms (dPRs). Here, we further the discussion initialized in \cite{Stojnicphretreal24} and consider precise analytical characterizations of spectral methods optimal overlaps and the impact they have on phase retrieval PTs. While overlaps themselves can be characterized via spectral random matrix theory methods (see, \cite{LuL17,LuoAL19,LuLi20,MondelliM19,MaDXMW21,DudejaB0M20}), we below present a different, simpler, and more general method that relies on \emph{Random duality theory} (RDT). This is also in line with the presentation of \cite{Stojnicphretreal24} where the RDT was heavily utilized for the analyses of dPR's ${\mathcal {PM}}$s.

\section{Spectral initializers}
\label{sec:specinit}

There are many ways how one can obtain a (non-exact) estimate of $\bar{\x}$. For example, to run Wirtinger flow, \cite{CandesLS15} suggested utilizing, as a preprocessing step, the so-called spectral estimators of the following form
 \begin{equation}
\bl{\mbox{\emph{\textbf{General sepctral initializer:}}}} \qquad  \x^{(spec)}  \triangleq  \mbox{max eigenvector} \lp A^T \mbox{diag} \lp  {\mathcal T}(\bar{\y})  \rp   A\rp,
\label{eq:speceq1}
\end{equation}
where ${\mathcal T} (\bar{\y})$ is a function of phaseless measurements $\bar{\y}$. Moreover, for the concreteness, \cite{CandesLS15} suggested the following
${\mathcal T} (\bar{\y})$ choice
 \begin{equation}
{\mathcal T} (\bar{\y}) =\bar{\y} (=|A\x|^2).
\label{eq:speceq2}
\end{equation}
Numerical implementations conducted in \cite{CandesLS15} confirmed that the above choice works rather well. Moreover, \cite{CandesLS15} also recognized that the following quantity (the above mentioned so-called overlap between $\x^{(spec)}$ and the true signal, $\bar{\x}$) might critically impact the ability of Wirtinger flow to actually solve the phase retrieval
 \begin{equation}
\hspace{-1.2in}\bl{\mbox{\emph{\textbf{Spectral initializer overlap:}}}} \qquad x^{(spec)} \triangleq  \frac{\lp \x^{(spec)} \rp^T \bar{\x}  }  {\|\x^{(spec)}\|_2 \|\bar{\x}\|_2}.
\label{eq:speceq3}
\end{equation}
Combining such an observation with the prevalent intuition regarding favorable larger starting points overlaps, one can position the following (preprocessing) problem as of key importance
 \begin{equation}
\hspace{-1.2in}\bl{\mbox{\emph{\textbf{Overlap optimal spectral initializer (OptSpin):}}}} \qquad \hat{{\mathcal T}}(\bar{\y}) \triangleq   \mbox{arg}\max_{{\mathcal T}(\bar{\y})} x^{(spec)}.
\label{eq:speceq4}
\end{equation}
Consequently, the following definition is rather natural as well
 \begin{equation}
\hspace{-1.2in}\bl{\mbox{\emph{\textbf{Optimal overlap:}}}} \qquad  \hat{x}^{(spec)}\triangleq   \max_{{\mathcal T}(\bar{\y})} x^{(spec)}.
\label{eq:speceq5}
\end{equation}

\subsection{Analysis of spectral method overlaps via Random Duality Theory (RDT)}
\label{sec:ubrdt}

Following the initial success of \cite{CandesLS15}, subsequent works \cite{LuL17,LuoAL19,LuLi20,MondelliM19,MaDXMW21,DudejaB0M20}) recognized that, when it comes to making $x^{(spec)}$ as large as possible,  there are actually much better choices for ${\mathcal T} (\bar{\y})$ than the simple one suggested in  (\ref{eq:speceq2}). We below present a way to determine ${\mathcal T} (\bar{\y})$ that actually maximizes $x^{(spec)}$, i.e., we present a way to
 determine $\hat{{\mathcal T}}(\bar{\y})$. As the methods that we present are developed through the utilization of the Random duality theory (RDT), before proceeding with the technical presentation we briefly summarize the key principles behind the RDT machinery developed in a long series of work \cite{StojnicCSetam09,StojnicICASSP10var,StojnicISIT2010binary,StojnicICASSP10block,StojnicRegRndDlt10,StojnicGenLasso10}. We then continue by discussing concrete implementation of each of these principles within the optimal spectral initializers context of interest here.

\vspace{-.0in}\begin{center}
 \begin{tcolorbox}[title={\small Summary of the RDT's main principles} \cite{StojnicCSetam09,StojnicRegRndDlt10}]
\vspace{-.15in}
{\small \begin{eqnarray*}
 \begin{array}{ll}
\hspace{-.19in} \mbox{1) \emph{Finding underlying optimization algebraic representation}}
 & \hspace{-.0in} \mbox{2) \emph{Determining the random dual}} \\
\hspace{-.19in} \mbox{3) \emph{Handling the random dual}} &
 \hspace{-.0in} \mbox{4) \emph{Double-checking strong random duality.}}
 \end{array}
  \end{eqnarray*}}
\vspace{-.2in}
 \end{tcolorbox}
\end{center}\vspace{-.0in}

\noindent To make sure that the presentation is as neat as possible, we formulate all key results (including both simple and more complicated ones) as lemmas or theorems.

\vspace{.1in}

\noindent \underline{1) \textbf{\emph{Algebraic spectral initializers characterization:}}} We start by noting that $ \x^{(spec)} $ defined in (\ref{eq:speceq1}) can alternatively be obtained as
\begin{equation}
 \x^{(spec)} =\mbox{max eigenvector} \lp A^T \mbox{diag} \lp  {\mathcal T}(\bar{\y})  \rp   A\rp = \mbox{arg} \max_{\|\x\|_2=1} \x^T\lp A^T \mbox{diag} \lp  {\mathcal T}(\bar{\y})  \rp   A\rp \x.
\label{eq:rdteq0a0a0}
\end{equation}
For the completeness, we also define the corresponding maximum eigenvalue
\begin{equation}
 \lambda^{(spec)} =\mbox{max eigenvalue} \lp A^T \mbox{diag} \lp  {\mathcal T}(\bar{\y})  \rp   A\rp = \max_{\|\x\|_2=1} \x^T\lp A^T \mbox{diag} \lp  {\mathcal T}(\bar{\y})  \rp   A\rp \x.
\label{eq:rdteq0a0a0a0}
\end{equation}
From (\ref{eq:rdteq0a0a0}), we further have
\begin{eqnarray}
 \x^{(spec)} = \mbox{arg}\max_{\|\x\|_2=1,\z}  & &  \z^T \mbox{diag} \lp  {\mathcal T}(\bar{\y})  \rp   \z \nonumber \\
 \mbox{subject to } & & A\x=\z.
\label{eq:rdteq0a0a1}
\end{eqnarray}
One then also observes that rotational invariance of $A$ allows any ($A$ independent) rotation of $\bar{\x}$. Following into the footsteps of \cite{Stojnicphretreal24}, we choose the rotation that transform $\bar{\x}$ into  $\bar{\x}=[\|\bar{\x}\|_2,0,\dots,0 ]^T$. This immediately implies
\begin{eqnarray}
 \bar{\y}=A\bar{\x}=A_{:,1}\|\bar{\x}\|_2
\label{eq:rdteq0a0a2}
\end{eqnarray}
and allows to rewrite (\ref{eq:rdteq0a0a1}) as
\begin{eqnarray}
 \x^{(spec)} = \mbox{arg}\max_{\|\x\|_2=1,\z}  & &  \z^T \mbox{diag} \lp  {\mathcal T}(A_{:,1}\|\bar{\x}\|_2)  \rp   \z \nonumber \\
 \mbox{subject to } & & A\x=\z,
\label{eq:rdteq0a0a3}
\end{eqnarray}
where $A_{:,i}$ stands for the $i$-th column of $A$. Analogously to (\ref{eq:ex1a4a0}), we find it convenient to introduce the following \emph{fundamental spectral initializer}
\begin{eqnarray}
\hspace{-.2in}\bl{\mbox{\emph{\textbf{(f-spin) optimization:}}}} \qquad  \xi_s(x) \triangleq \max_{\|\x\|_2=1,\z}  & &  \z^T \mbox{diag} \lp  {\mathcal T}(A_{:,1}\|\bar{\x}\|_2)  \rp   \z \nonumber \\
 \mbox{subject to } & & A\x=\z \nonumber \\
  & & \x^T\bar{\x}=\x_1\|\bar{\x}\|_2=x.
\label{eq:rdteq0a0a3a0}
\end{eqnarray}
After setting $r\triangleq \sqrt{1-x^2}$ and noting that restriction to $\|\bar{\x}\|_2=1$ brings no loss of generality, we further have
 \begin{eqnarray}
 - \xi_s(x) = \min_{\x,\z} & & -\z^T \mbox{diag} \lp  {\mathcal T}(A_{:,1})  \rp   \z \nonumber \\
  \mbox{subject to} & &  A\x=A_{:,1}x + A_{:,2:n}\x_{2:n} = \z \nonumber \\
   & & \sum_{i=2}^{n}\x_i^2=1-x^2=r^2, \label{eq:rdteq0a1}
\end{eqnarray}
Writing Lagrangian then gives
 \begin{eqnarray}
 - \xi_s(x) = \min_{\x,\z} \max_{\y}   & &
 -\z^T \mbox{diag} \lp  {\mathcal T}(A_{:,1})  \rp   \z
  +\y^TA_{:,1}x + \y^TA_{:,2:n}\x_{2:n} -\y^T \z \nonumber \\
  \mbox{subject to}
     & & \sum_{i=2}^{n}\x_i=c-x^2=r^2. \label{eq:rdteq0a2}
\end{eqnarray}
We also set $\g^{(0)}=A_{:,1}$ and rewrite (\ref{eq:rdteq0a2}) in a more compact way
 \begin{eqnarray}
 - \xi_s(x) = \min_{\|\x_{2:n}\|_2=r,\z} \max_{\y}  \lp
  -\z^T \mbox{diag} \lp  {\mathcal T}( \g^{(0)} )  \rp   \z
  +\y^T \g^{(0)}x + \y^TA_{:,2:n}\x_{2:n} -\y^T \z \rp. \label{eq:rdteq0a3}
\end{eqnarray}
Since the above  is a fairly useful characterization of the spectral initializers, it is summarized in the following lemma.

\begin{lemma} Consider real phase retrieval (PR) problem and, for a general componentwise-wise preprocessing function ${\mathcal T}(\cdot):\mR^m\rightarrow \mR^m$, let the associated general spectral initializer, $\x^{(spec)}$, be as in (\ref{eq:speceq1}). Assume a high-dimensional linear/proportional regime
with $\alpha=\lim_{n\rightarrow\infty}\frac{m}{n}$ and let $A\in\mR^{m\times n}$ and $\g^{(0)}\triangleq A_{:,1}$.  Set $r\triangleq \sqrt{1-x^2}$. Let
$\lambda^{(spec)}$ and $\x^{(spec)}$ be the maximum eigenvalue and the associated maximum eigenvector of $A^T \mbox{diag} \lp  {\mathcal T}(\bar{\y})  \rp   A$, i.e., let $\lambda^{(spec)}$ and $\x^{(spec)}$  be as defined in  (\ref{eq:rdteq0a0a0}) and (\ref{eq:rdteq0a0a0a0}), respectively. Moreover, let $x^{(spec)}$ be as in (\ref{eq:speceq3}).  Then
 \begin{eqnarray}\label{eq:ta10}
 \frac{\lambda^{(spec)}}{n} & = & \max_{x} \frac{\xi_s(x)}{n}  =   -\min_{x} f^{(s)}_{rp}(x;A) \nonumber \\
 x^{(spec)} = \x_1^{(spec)} & = & \arg \min_{x} f^{(s)}_{rp}(x;A) \nonumber \\
\x_{2:n}^{(spec)}  & = & \arg \min_{\|\x_{2:n}\|_2=r} \lp \min_{\z} \max_{\y}  \lp
  -\z^T \mbox{diag} \lp  {\mathcal T}(\g^{(0)} )  \rp   \z
  +\y^T \g^{(0)}x + \y^TA_{:,2:n}\x_{2:n} -\y^T \z \rp \rp, \nonumber \\
\end{eqnarray}
 where
\begin{equation}\label{eq:ta11}
f^{(s)}_{rp}(x;A)\triangleq \frac{1}{n}  \min_{\|\x_{2:n}\|_2=r,\z} \max_{\y}  \lp
  -\z^T \mbox{diag} \lp  {\mathcal T}(\g^{(0)} )  \rp   \z
  +\y^T \g^{(0)}x + \y^TA_{:,2:n}\x_{2:n} -\y^T \z \rp.
\end{equation}
    \label{lemma:lemma1}
\end{lemma}
\begin{proof}
Follows automatically from the above discussion after recognizing that (\ref{eq:ta11}) is cosmetically  scaled (\ref{eq:rdteq0a3}).
\end{proof}

The optimization program on the right hand side of (\ref{eq:ta11}) constitutes the so-called \emph{random primal}. The corresponding \emph{random dual} is discussed next.

\vspace{.1in}
\noindent \underline{2) \textbf{\emph{Determining the random dual:}}} Following the typical practice within the RDT, the concentration of measure is utilized as well. This basically means that for any fixed $\epsilon >0$,  one has (see, e.g. \cite{StojnicCSetam09,StojnicRegRndDlt10,StojnicICASSP10var})
\begin{equation*}
\lim_{n\rightarrow\infty}\mP_{A}\left (\frac{|f^{(s)}_{rp}(x;A)-\mE_{A}(f^{(s)}_{rp}(x;A)|}{\mE_{A}(f^{(s)}_{rp}(x;A)}>\epsilon\right )\longrightarrow 0.\label{eq:ta15}
\end{equation*}
The so-called random dual theorem that we state below is another fundamentally important ingredient of the RDT machinery.
\begin{theorem} Assume the setup of Lemma \ref{lemma:lemma1} with the elements of $A\in\mR^{m\times n}$ ($\g^{(0)}\in\mR^{m\times 1}$ and $A_{:,2:n}\in\mR^{m\times (n-1)}$), $\g^{(1)}\in\mR^{m\times 1}$, and  $\h^{(1)}\in\mR^{(n-1)\times 1}$  being iid standard normals. Let $x$ be a nonnegative  scalar ($0\leq x \leq 1$) and set $r\triangleq \sqrt{1-x^2}$. Let
\vspace{-.0in}
\begin{eqnarray}
\cG & \triangleq & \lp A,\g^{(1)},\h^{(1)}\rp = \lp\g^{(0)},A_{:,2:n},\g^{(1)},\h^{(1)}\rp  \nonumber \\
\phi^{(s)}(\x,\z,\y) & \triangleq &
 \lp
  -\z^T \mbox{diag} \lp  {\mathcal T}(\g^{(0)} )  \rp   \z
  +\y^T \g^{(0)}x   +  \y^T \g^{(1)}\|\x_{2:n}\|_2 + \lp \x_{2:n}  \rp^T\h^{(1)}\|\y\|_2  -\y^T  \z  \rp
\nonumber \\
 f^{(s)}_{rd}(x;\cG) & \triangleq &
\frac{1}{n}  \min_{\|\x_{2:n}\|_2=r,\z} \max_{\|\y\|_2=r_y,r_y>0}  \phi^{(s)}(\x,\z,\y)
  \nonumber \\
 \phi^{(s)}_0 & \triangleq & \lim_{n\rightarrow\infty} \mE_{\cG} f^{(s)}_{rd}(x;\cG).\label{eq:ta16}
\vspace{-.0in}\end{eqnarray}
One then has \vspace{-.02in}
\begin{eqnarray}
  \lim_{n\rightarrow\infty}\mP_{ A } (f^{(s)}_{rp} (x; A )   >  \phi^{(s)}_0)\longrightarrow 1.\label{eq:ta17a0}
\end{eqnarray}
  \label{thm:thm1}
\end{theorem}\vspace{-.17in}
\begin{proof}
Follows automatically after applying the Gordon's comparison theorem (see, e.g., Theorem B in \cite{Gordon88}) and  conditioning on $\g^{(0)}$. Gordon's theorem can be deduced as a special case of a series of results Stojnic obtained in \cite{Stojnicgscomp16,Stojnicgscompyx16} (see Theorem 1, Corollary 1, and Section 2.7.2 in \cite{Stojnicgscomp16} as well as Theorem 1, Corollary 1, and Section 2.3.2 in \cite{Stojnicgscompyx16}).
\end{proof}

\vspace{.1in}
\noindent \underline{3) \textbf{\emph{Handling the random dual:}}} To handle the above random dual, we rely on a series of techniques introduced in \cite{StojnicCSetam09,StojnicICASSP10var,StojnicISIT2010binary,StojnicICASSP10block,StojnicRegRndDlt10}. We first note that the optimizations over $\x$ and $\y$ in  (\ref{eq:ta16}) are relatively simple and allow to write
\begin{equation}
 f^{(s)}_{rd}(x;\cG) =
\frac{1}{n}  \min_{\z} \max_{r_y>0}
 \lp
  -\z^T \mbox{diag} \lp  {\mathcal T}(\g^{(0)} )  \rp   \z
  +\|\g^{(0)}x   +  \g^{(1)}r -\z\|_2 r_y - \|\h^{(1)} \|_2 r r_y \rp.
\label{eq:hrd1}
 \end{equation}
One then also has
\begin{eqnarray}
 f^{(s)}_{rd}(x;\cG)
  &  =  &
\frac{1}{n}  \min_{\z} \max_{r_y>0}
 \lp
   -\z^T \mbox{diag} \lp  {\mathcal T}(\g^{(0)} )  \rp   \z
  +\|\g^{(0)}x   +  \g^{(1)}r -\z\|_2^2 r_y - \|\h^{(1)} \|_2^2 r^2 r_y \rp \nonumber \\
  &  \geq   &
\frac{1}{n}  \max_{r_y>0} \min_{\z}
 \lp
  -\z^T \mbox{diag} \lp  {\mathcal T}(\g^{(0)} )  \rp   \z
  +\|\g^{(0)}x   +  \g^{(1)}r -\z\|_2^2 r_y - \|\h^{(1)} \|_2^2 r^2 r_y \rp
  \nonumber \\
  &  =   &
\frac{1}{n}  \max_{r_y>0} \min_{\z_i}
 \lp
\sum_{i=1}^{m}
\lp
-\tau_i\z_i^2  +\|\g_i^{(0)}x   +  \g_i^{(1)}r -\z_i\|_2^2 r_y \rp
  - \|\h^{(1)} \|_2^2 r^2 r_y \rp,
\label{eq:hrd2}
 \end{eqnarray}
 where we set
\begin{eqnarray}
\tau =  {\mathcal T}(\g^{(0)} ) .
\label{eq:hrd2a0}
 \end{eqnarray}
Statistical identicalness over $i$ together with concentrations then gives
\begin{eqnarray}
 \phi^{(s)}_0 & \triangleq & \lim_{n\rightarrow\infty} \mE_{\cG} f^{(s)}_{rd}(\cG)
\geq
  \max_{r_y>0} \mE_{\cG}  \min_{\z_i} \cL_1(r_y),
\label{eq:hrd6}
 \end{eqnarray}
where
\begin{eqnarray}
\cL_1(r_y)
  &  = &
\alpha  \lp
 -\tau_i \z_i^2
  + \lp \g_i^{(0)}x   +  \g_i^{(1)}r -\z_i \rp^2 r_y \rp
  -  r^2 r_y.
\label{eq:hrd7}
 \end{eqnarray}
When  $\r_y> \tau_i$ finding the following derivative turns out to be useful
\begin{eqnarray}
\frac{d\cL_1(r_y)}{d\z_i}
  &  = &
\alpha \lp -2\tau_i\z_i
  -2\lp \g_i^{(0)}x   +  \g_i^{(1)}r -\z_i \rp  r_y\rp.
\label{eq:hrd7a0}
 \end{eqnarray}
One can then equal the above derivative to zero while keeping in mind the condition $\r_y\geq \tau_i$ and determine the optimal $\z_i$ as
\begin{eqnarray}
\hat{\z}_i= \begin{cases}
              \frac{r_y}{-\tau_i+r_y} \lp \g_i^{(0)}x   +  \g_i^{(1)}r\rp
, & \mbox{if } r_y> \tau_i \\
              0, & \mbox{otherwise}.
            \end{cases}
\label{eq:hrd7a1}
 \end{eqnarray}
A combination of  (\ref{eq:hrd7}) and  (\ref{eq:hrd7a1}) further gives
\begin{eqnarray}
\min_{\z_i} \cL_1(r_y)
     &  = &
     \begin{cases}
       -\alpha
\frac{\tau_i r_y}{-\tau_i+r_y} \lp \g_i^{(0)}x   +  \g_i^{(1)}r  \rp^2
  -  r^2 r_y, & \mbox{if } r_y>\tau_i \\
              \alpha
r_y \lp \g_i^{(0)}x   +  \g_i^{(1)}r  \rp^2
  -  r^2 r_y, & \mbox{otherwise}.
     \end{cases}
\label{eq:hrd7a2}
 \end{eqnarray}
 For the time being we assume $r_y>\tau_i$ and after setting
\begin{eqnarray}
f_q & \triangleq & - \mE_{\cG} \frac{\tau_i r_y}{-\tau_i+r_y}   \lp \g_i^{(0)}x   +  \g_i^{(1)}r  \rp^2,
\label{eq:hrd7a2a0}
 \end{eqnarray}
 easily write
\begin{eqnarray}
 \mE_{\cG} \min_{\z_i} \cL_1(r_y)
       &  = &
\alpha
   f_q  -  r^2 r_y.
\label{eq:hrd7a3}
 \end{eqnarray}
After plugging this back in (\ref{eq:hrd6}) one has
\begin{eqnarray}
 \phi^{(s)}_0 & \triangleq & \lim_{n\rightarrow\infty} \mE_{\cG} f^{(s)}_{rd}(\cG)
    \geq
  \max_{r_y>0}
  \lp
  \alpha
\ f_q
  -  r^2 r_y \rp.
\label{eq:hrd7a4}
 \end{eqnarray}
To further analyze the optimization problem on the right hand side
one computes the derivative of the expression under $\max$
\begin{equation}
 \frac{d  \lp
  \alpha
 f_q
  -  r^2 r_y \rp}{dr_y} =   \alpha
 \frac{df_q}{dr_y}
  -  r^2
  =   -\alpha
 \frac{d\lp \mE_{\cG} \frac{\tau_i r_y}{-\tau_i+r_y}   \lp \g_i^{(0)}x   +  \g_i^{(1)}r  \rp^2 \rp}{dr_y}
  -  r^2
    =   \alpha
 \mE_{\cG} \frac{\tau_i^2   \lp \g_i^{(0)}x   +  \g_i^{(1)}r  \rp^2 }{(-\tau_i+r_y)^2}
  -  r^2.
\label{eq:hrd7a5}
 \end{equation}
Keeping in mind that $\tau_i$ depends only on $\g_i^{(0)}$ (and not on $\g_i^{(1)}$), one can integrate out $\g_i^{(1)}$ and rewrite (\ref{eq:hrd7a5}) as
\begin{equation}
 \frac{d  \lp
  \alpha
 f_q
  -  r^2 r_y \rp}{dr_y}     =   \alpha
 \mE_{\cG} \frac{\tau_i^2   \lp \lp \g_i^{(0)}\rp^2(1-r^2)   + r^2  \rp }{(-\tau_i+r_y)^2}
  -  r^2.
\label{eq:hrd7a5a0}
 \end{equation}
As (\ref{eq:ta10}) and (\ref{eq:ta17a0}) state, an additional minimization over $r$ of the righthand side of (\ref{eq:hrd7a4}) is ultimately needed as well. To that end we find it useful to complement the above $r_y$ derivative with its $r$-derivative analogue. Since $r^2$ appears as a variable rather than $r$, formally speaking,  we find it more convenient to compute the $r^2$-derivative
\begin{equation}
 \frac{d  \lp
  \alpha
 f_q
  -  r^2 r_y \rp}{dr^2} =   \alpha
 \frac{df_q}{dr^2}
  -  r_y
  =   -\alpha
 \frac{d\lp \mE_{\cG}\lp \frac{\tau_i r_y}{-\tau_i+r_y}   \lp \g_i^{(0)}x   +  \g_i^{(1)}r  \rp^2 \rp\rp}{dr^2}
  -  r_y
    =   \alpha
 \mE_{\cG} \frac{\tau_i r_y  \lp  \lp \g_i^{(0)} \rp^2  -  1  \rp }{-\tau_i+r_y}
  -  r_y,
\label{eq:hrd7a5a0a0}
 \end{equation}
where the second equality follows after recognizing once again that $\g_i^{(1)}$ can be integrated out. Equalling (\ref{eq:hrd7a5a0}) and (\ref{eq:hrd7a5a0a0}) to zero is then sufficient to determine optimal $\hat{r}_y$ and $\hat{r}$ and provided that $\hat{r}_y>\tau_i$ it is also sufficient to determine the tightest (with respect to $r$) right hand side of (\ref{eq:hrd7a4}). Assuming that all of this indeed happens for a choice of $\tau_i$, one still needs to determine the optimal choice of $\tau_i$ that makes $\hat{r}$ as small as possible (that would automatically make $x^{(spec)}$ as large as possible and effectively determine $\hat{x}^{(spec)}$ in (\ref{eq:speceq5})). We now show how this is done. Let $\tau_i$, $\hat{r}_y$, and $\hat{r}$ be such that the derivatives in (\ref{eq:hrd7a5a0}) and (\ref{eq:hrd7a5a0a0}) are indeed zero.  From (\ref{eq:hrd7a5a0}) we have
\begin{equation}
\frac{1}{\hat{r}^2}   =
\frac{\frac{1}{\alpha} + \mE_{\cG} \lp \frac{\tau_i^2   \lp \lp \g_i^{(0)}\rp^2   -1  \rp } {(-\tau_i+r_y)^2} \rp }  { \mE_{\cG} \lp  \frac{\tau_i^2  \lp \g_i^{(0)}\rp^2  }{(-\tau_i+r_y)^2}   \rp} =
1+\frac{\frac{1}{\alpha} - \mE_{\cG} \lp \frac{\tau_i^2  } {(-\tau_i+\hat{r}_y)^2} \rp }  { \mE_{\cG} \lp  \frac{\tau_i^2  \lp \g_i^{(0)}\rp^2  }{(-\tau_i+\hat{r}_y)^2}   \rp}.
\label{eq:hrd7a5a0a1}
 \end{equation}
and from (\ref{eq:hrd7a5a0a0})
\begin{equation}
 \mE_{\cG} \lp \frac{\tau_i  \lp  \lp \g_i^{(0)} \rp^2  -  1  \rp }{-\tau_i+\hat{r}_y} \rp
  = \frac{1}{\alpha}.
\label{eq:hrd7a5a0a2}
 \end{equation}
After setting
\begin{equation}
\zeta = \frac{\tau_i  }{-\tau_i+\hat{r}_y}, \quad  b=  \lp \g_i^{(0)} \rp^2  -  1, \quad a= \lp \g_i^{(0)} \rp^2,
\label{eq:hrd7a5a0a3}
 \end{equation}
 minimal $\hat{r}$ will be obtained as
\begin{equation}
\hat{r}=\frac{1}{\sqrt{1+\hat{d}}},
\label{eq:hrd7a5a0a3a0}
 \end{equation}
 where
\begin{eqnarray}
\hat{d} =\max_{\zeta}  & &  \frac{\frac{1}{\alpha}  - \mE_{\cG} \lp \zeta^2 \rp } {\mE_{\cG} \lp \zeta^2 a \rp} \nonumber \\
\mbox{subject to} & & \mE_{\cG} \lp \zeta b \rp =\frac{1}{\alpha},
  \label{eq:hrd7a5a0a4}
 \end{eqnarray}
 The above is a simple linearly constrained quasi-convex program that can be solved in a closed form. One can proceed by keeping integral notation. However, to be more aligned with a common optimization practice we set
\begin{eqnarray}
\theta= \frac{e^{-\frac{\lp \g_i^{(0)}\rp^2}{2}}d\g_i^{(0)}}{\sqrt{2\pi}},
  \label{eq:hrd7a5a0a4a0}
 \end{eqnarray}
assume $d\g_i^{(0)}\rightarrow 0$ and write a discretized analogue of (\ref{eq:hrd7a5a0a4})
\begin{eqnarray}
\hat{d} = \lim_{l\rightarrow\infty}\max_{\zeta}  & &  \frac{\frac{1}{\alpha}  - \sum_{j=1}^{l} \zeta_j^2\theta_j } {\sum_{j=1}^{l} \zeta_j^2 a_j\theta_j} \nonumber \\
\mbox{subject to} & & \sum_{j=1}^{l}\zeta_j b_j\theta_j =\frac{1}{\alpha}.
  \label{eq:hrd7a5a0a4a1}
 \end{eqnarray}
 A cosmetic change gives
\begin{eqnarray}
\hat{d} = \lim_{l\rightarrow\infty}\max_{\zeta}  & & d \nonumber \\
\mbox{subject to}  & &
 \frac{1}{\alpha}  - \sum_{j=1}^{l} \zeta_j^2\theta_j  \geq d \sum_{j=1}^{l} \zeta_j^2 a_j\theta_j \nonumber \\
& & \sum_{j=1}^{l}\zeta_j b_j\theta_j =\frac{1}{\alpha}.
  \label{eq:hrd7a5a0a4a2}
 \end{eqnarray}
 Clearly, optimal $d$ (i.e., $\hat{d}$) will be obtained if
\begin{eqnarray}
 \frac{1}{\alpha}  = \lim_{l\rightarrow\infty}\min_{\zeta}  & &
  \sum_{j=1}^{l} \zeta_j^2 (\hat{d}a_j+1)\theta_j
  \nonumber \\
\mbox{subject to} & & \sum_{j=1}^{l}\zeta_j b_j\theta_j =\frac{1}{\alpha}.
  \label{eq:hrd7a5a0a4a2}
 \end{eqnarray}
The above is a simple linearly constrained quadratic program. To solve it, we write the Lagrangian
\begin{eqnarray}
\cL_s(\nu) =    \sum_{j=1}^{l} \zeta_j^2 (\hat{d}a_j+1)\theta_j -\nu\sum_{j=1}^{l}\zeta_j b_j\theta_j +\nu\frac{1}{\alpha}.
 \label{eq:hrd7a5a0a4a3}
 \end{eqnarray}
Strong duality then gives
\begin{equation}
 \frac{1}{\alpha}  = \lim_{l\rightarrow\infty}\min_{\zeta} \max_{\nu}  \cL_s(\nu)= \lim_{l\rightarrow\infty}  \max_{\nu}\min_{\zeta}  \cL_s(\nu)= \lim_{l\rightarrow\infty}  \max_{\nu}\min_{\zeta}
  \lp \sum_{j=1}^{l} \zeta_j^2 (\hat{d}a_j+1)\theta_j -\nu\sum_{j=1}^{l}\zeta_j b_j\theta_j +\nu\frac{1}{\alpha}\rp.
  \label{eq:hrd7a5a0a4a4}
 \end{equation}
Taking derivative with respect to $\zeta_j$ gives
\begin{eqnarray}
 \frac{d\cL_s(\nu)}{\d\zeta_j} = 2\zeta_j (\hat{d}a_j+1)\theta_j -\nu b_j\theta_j.
  \label{eq:hrd7a5a0a4a5}
 \end{eqnarray}
Equaling the derivative to zero and keeping in mind that $\theta_j\geq 0$, one finds for the optimal $\zeta_j$
\begin{eqnarray}
\hat{\zeta}_j = \frac{\nu b_j}{2(\hat{d}a_j+1)}.
  \label{eq:hrd7a5a0a4a6}
 \end{eqnarray}
Plugging the above $\hat{\zeta}$ in (\ref{eq:hrd7a5a0a4a4}) further gives
\begin{equation}
 \frac{1}{\alpha}    = \lim_{l\rightarrow\infty}  \max_{\nu} \lp -\nu^2\sum_{j=1}^{l} \frac{b_j^2\theta_j}{4(\hat{d}a_j+1)} +\nu\frac{1}{\alpha}\rp.
  \label{eq:hrd7a5a0a4a7}
 \end{equation}
After solving the residual maximization, one obtains for the optimal $\nu$
\begin{equation}
\hat{\nu} = \frac{1}{2}\frac{\frac{1}{\alpha}}{ \lim_{l\rightarrow\infty} \sum_{j=1}^{l} \frac{b_j^2\theta_j}{4(\hat{d}a_j+1)} }.
  \label{eq:hrd7a5a0a4a8}
 \end{equation}
Plugging the above $\hat{\nu}$ in (\ref{eq:hrd7a5a0a4a7}) then gives
\begin{equation}
 \frac{1}{\alpha}
 =
 \frac{1}{4}\frac{\frac{1}{\alpha^2}}{ \lim_{l\rightarrow\infty} \sum_{j=1}^{l} \frac{b_j^2\theta_j}{4(\hat{d}a_j+1)} },
  \label{eq:hrd7a5a0a4a9}
 \end{equation}
which, after returning to integral representation, means that optimal $d$, i.e., $\hat{d}$ is such that
\begin{equation}
 \lim_{l\rightarrow\infty}  \sum_{j=1}^{l} \frac{b_j^2\theta_j}{(\hat{d}a_j+1)}
 =
\frac{1}{\alpha} \quad  \Longleftrightarrow \quad  \mE_{\cG} \frac{\lp\lp \g_i^{(0)} \rp^2  -  1\rp^2}{\hat{d}\lp \g_i^{(0)} \rp^2  +  1}
 =
\frac{1}{\alpha}.
  \label{eq:hrd7a5a0a4a10}
 \end{equation}
Moreover, a combination of (\ref{eq:hrd7a5a0a4a7}) and (\ref{eq:hrd7a5a0a4a9}) gives
\begin{equation}
\hat{\nu} =2.  \label{eq:hrd7a5a0a4a11}
 \end{equation}
Plugging this value for optimal $\nu$ in (\ref{eq:hrd7a5a0a4a7}) we also find the optimal form of $\zeta$ (without the residual dependence on $\nu$)
\begin{eqnarray}
\hat{\zeta}_j = \frac{b_j}{(\hat{d}a_j+1)}.
  \label{eq:hrd7a5a0a4a12}
 \end{eqnarray}
Combining (\ref{eq:hrd7a5a0a4a12}) and (\ref{eq:hrd7a5a0a3}) and switching back to integral form one has that optimal $\tau_i$, $\hat{\tau}_i$ satisfies
\begin{equation}
 \frac{\hat{\tau}_i  }{-\hat{\tau}_i+\hat{r}_y} = \frac{\lp \g_i^{(0)} \rp^2  -  1}{(\hat{d}\lp \g_i^{(0)} \rp^2+1)},
\label{eq:hrd7a5a0a4a13}
 \end{equation}
or alternatively
\begin{equation}
\hat{\tau}_i = \frac{\hat{r}_y}{\hat{d}+1}\frac{\lp \g_i^{(0)} \rp^2  -  1}{\lp \g_i^{(0)} \rp^2},
\label{eq:hrd7a5a0a4a14}
 \end{equation}
Recalling on (\ref{eq:hrd2a0}) we obtain for the optimal preprocessing function ${\mathcal T}(\cdot)$
\begin{eqnarray}
 \hat{{\mathcal T}}(\g^{(0)} ) =\frac{\hat{r}_y}{\hat{d}+1}\frac{\lp \g_i^{(0)} \rp^2  -  1}{\lp \g_i^{(0)} \rp^2}.
\label{eq:hrd7a5a0a4a15}
 \end{eqnarray}
 One should note that since  $\hat{{\mathcal T}}(\cdot)$ is applied on $\bar{\y}$ (only measurements $\bar{\y}$ are available in phase retrieval problem) and since $\bar{\y}=0$ with probability 1 in the context that we consider here, potential singularities have no effect on any of the above calculations and associated algorithmic implementations. Also, as stated earlier (right after (\ref{eq:hrd7a2})), the above holds assuming $\hat{r}_y>\hat{\tau}_i$. Since scaling $ \hat{{\mathcal T}}(\g^{(0)} ) $ by any constant does not affect the optimal $\hat{x}^{(spec)}$, one can choose $r_y=\hat{d}+1$, where $\hat{d}$ is the solution of   (\ref{eq:hrd7a5a0a4a10}), which then gives
\begin{eqnarray}
 \hat{{\mathcal T}}(\g^{(0)} ) = \frac{\lp \g_i^{(0)} \rp^2  -  1}{\lp \g_i^{(0)} \rp^2}
 = 1 - \frac{ 1}{\lp \g_i^{(0)} \rp^2}\leq 1+ \hat{d} =\hat{r}_y,
\label{eq:hrd7a5a0a4a16}
 \end{eqnarray}
 which ensures that all of the above is indeed correct. Moreover, after introducing $\hat{\gamma}=\hat{d}+1$, one from (\ref{eq:hrd7a5a0a3a0}) and  (\ref{eq:hrd7a5a0a4a10}) finds the optimal overlap
\begin{eqnarray}
\hat{x}^{(spec)}=\sqrt{1-\frac{1}{\hat{\gamma}}},
\label{eq:hrd7a5a0a4a17}
 \end{eqnarray}
where $\hat{\gamma}$ is the solution of
\begin{eqnarray}
 \frac{\hat{\gamma}}{(\hat{\gamma}-1)^3}\left (\frac{\hat{\gamma}}{2} e^{\frac{1}{2(\hat{\gamma}-1)}}\sqrt{2(\hat{\gamma}-1)\pi}\mbox{erfc}\left ( \frac{1}{\sqrt{2(\hat{\gamma}-1)}} \right ) -(\hat{\gamma}-1) \right)=\frac{1}{\alpha}.
\label{eq:hrd7a5a0a4a18}
 \end{eqnarray}

It is not that difficult to check that $ \hat{{\mathcal T}}(\cdot)$  from (\ref{eq:hrd7a5a0a4a16})  precisely matches the optimal choice  suggested in \cite{MondelliM19} and proven for any $\alpha$ in \cite{LuoAL19} (optimality of $\hat{{\mathcal T}}(\cdot)$ was proven for $\alpha=\frac{1}{2}$ in the limiting scenario corresponding to the so-called \emph{weak threshold} of spectral methods in \cite{MondelliM19} as well). It should be noted that the methodology that we presented above is very generic. It can be used to incorporate various forms of further structuring of unknown signal $\x$. This is done by following line by line practice developed within compressed sensing context in \cite{StojnicCSetam09,StojnicICASSP10block,StojnicICASSP10knownsupp,StojnicICASSP10var,Stojnicbinary16asym,StojnicISIT2010binary}. Among others, examples of structuring include sparse, block-sparse, binary, partially observed and many other signals. Spectral methods, on the other hand, rely on  powerful tools from random matrix or free probability theory and as such are often restricted to eigenvalues/eigenvectors related considerations.

  \vspace{.1in}
\noindent \underline{4) \textbf{\emph{Double checking the strong random duality:}}} As just stated, the above results match the optimal ones from \cite{LuoAL19} and one has that the strong random duality is in place as well (i.e., no further RDT lifting is needed).

The above spectral initializers are overlap optimal, i.e., they produce the smallest $r$ and the largest $x^{(spec)}$.
Numerical results for the optimal overlap are shown (as a function of sample complexity ratio) in Figure \ref{fig:fig1}.  In the following section we discuss how the optimal spectral initializers relate to phase transitions of \emph{descending} phase retrieval algorithms studied in companion paper \cite{Stojnicphretreal24}. To illustrate the relation we find it useful to select two particular points from the figure, $(\alpha,x)=(1.4,0.6439)$ and $(\alpha,x)=(1.6,0.7055)$.

\begin{figure}[h]
\centering
\centerline{\includegraphics[width=1\linewidth]{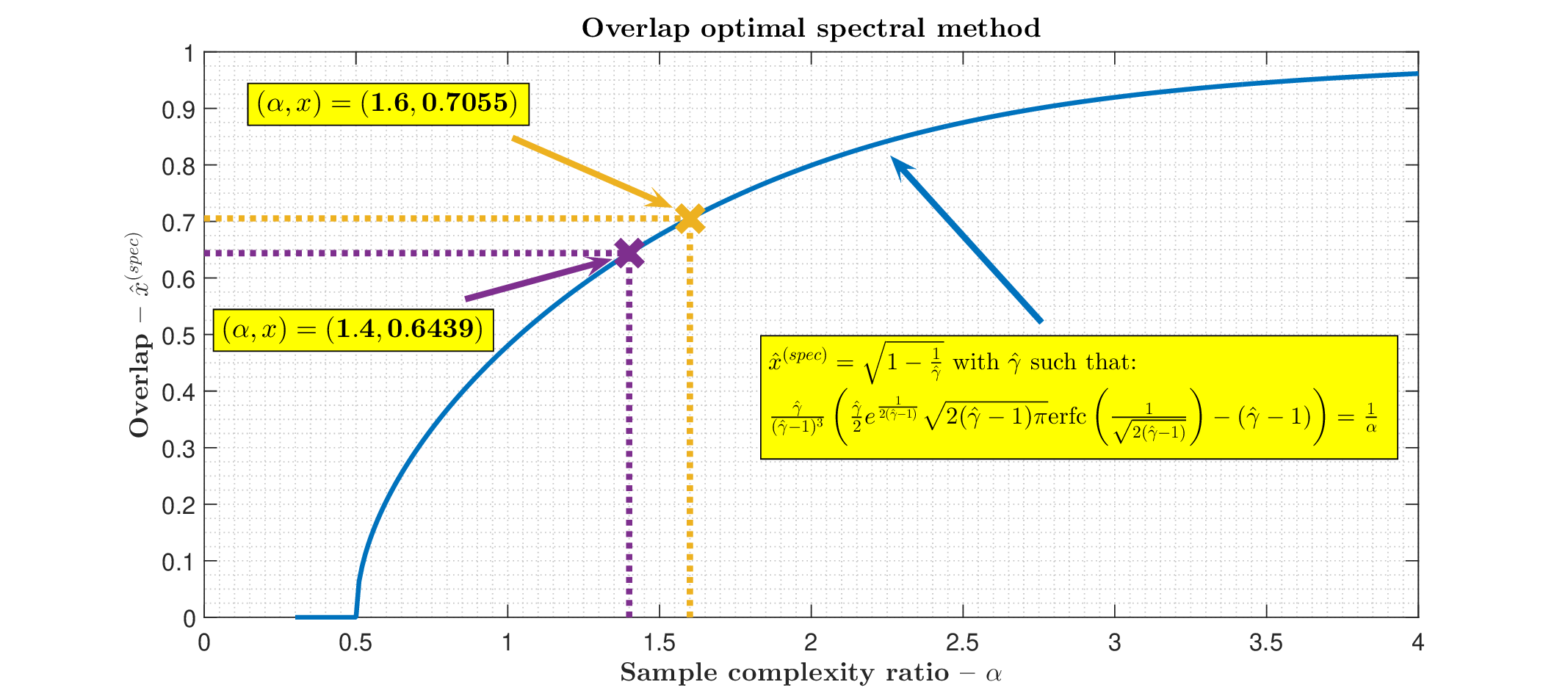}}
\caption{Optimal spectral initializers overlap $\hat{x}^{(spec)}$ a function of $\alpha$}
\label{fig:fig1}
\end{figure}

\section{Optimal spectral initializers and phase retrieval phase transitions}
\label{sec:algimp}

We first recall on key phase retrieval phase transition results obtained in \cite{Stojnicphretreal24}.

\begin{theorem} Consider the fundamental phase retrieval optimization (f-pro) from (\ref{eq:ex1a4a0}). Assume a high dimensional linear regime with $\alpha=\lim_{\rightarrow\infty} \frac{m}{n}$ and let the elements of $A\in\mR^{m\times n}$  be iid standard normals. For two positive scalars $c$ and $x$  ($0\leq x \leq c$), let $\xi(c,x)$ be as in (\ref{eq:ex1a4a0}). Set $r\triangleq \sqrt{c-x^2}$. One then has \vspace{-.0in}
\begin{eqnarray}
  \lim_{n\rightarrow\infty}\mP_{ A } \lp \frac{\xi(c,x)}{n}  \geq  \bar{\phi}_0  \geq   \phi_0 \rp \longrightarrow 1,\label{eq:ta17a0}
\end{eqnarray}
where $\phi_0$ and $\bar{\phi}_0$ are the so-called plain and lifted RDT estimates. In particular, for $\phi_0$ we have
\begin{eqnarray}
 \phi_0 = \max\lp \sqrt{\alpha f_q}
  -  r,0\rp^2.
\label{eq:thm1ahrd7a7}
 \end{eqnarray}
where
\begin{eqnarray}
 f_q   =  \frac{2}{\sqrt{2\pi}} \int_{0}^{\infty} \lp  (1+c) + 2A \erf\lp \frac{C}{\sqrt{2}}\rp   - 2B \frac{e^{-\frac{C^2}{2}}}{\sqrt{2\pi}}   \rp e^{-\frac{g_1^2}{2}}dg_1.
 \label{eq:thm1ahrd7a9}
\end{eqnarray}
and \begin{eqnarray}
A  =   g_1^2\sqrt{c-r^2}, \quad B  =  g_1r, \quad C  =  -g_1\frac{\sqrt{c-r^2}}{r}.
\label{eq:thm1ahrd7a8}
\end{eqnarray}
Moreover, for $\bar{\phi}_0$ we have
\begin{eqnarray}
 \bar{\phi}_0
& = &
 \max_{c_3> 0}  \max_{r_y>0}\min_{\gamma>0}
\Bigg .\Bigg(
 \frac{c_3}{2} r^2r_y^2 + \gamma
 -\frac{\alpha}{c_3} \log \lp f_{q}^{(lift)}\rp
 - \hat{\gamma}_{sph}  +\frac{1}{2c_3} \log \lp  1  -  \frac{c_3rr_y  }{2\hat{\gamma}_{sph}}     \rp
\Bigg.\Bigg),
\label{eq:thm1aplhrd16a0}
 \end{eqnarray}
where
\begin{eqnarray}
f_{q}^{(lift)}     =
\int_{-\infty}^{\infty}
\lp \frac{e^{\bar{D} +    \frac{\bar{A}.^2}{2\bar{C}}   }   }{2\sqrt{\bar{C}}} \erfc\lp \frac{ \bar{A}+\bar{C}\bar{B} } {\sqrt{2\bar{C}} }  \rp
 +
  \frac{e^{\bar{D}_2  +    \frac{\bar{A}_2.^2}{2\bar{C}_2}   }   }{2\sqrt{\bar{C}_2}} \erfc \lp -\frac{ \bar{A}_2+\bar{C}_2\bar{B}_2 } {\sqrt{2\bar{C}_2} }  \rp
  \rp \frac{e^{-\frac{ \lp g_1 \rp^2     } {2}  }}{\sqrt{2\pi}} dg_1,
  \label{eq:thm1aplhrd11}
 \end{eqnarray}
and
\begin{eqnarray} \label{eq:thm1aplhrd13}
\hat{\gamma}_{sph} & = &\frac{c_3rr_y+\sqrt{c_3^2r^2r_y^2+4}}{4}, \quad  \gamma_x  =  \frac{\bar{r}_y}{1+\bar{r}_y}\nonumber \\
     \bar{A} & = & -2c_3\gamma_x r |g_1| + 2c_3\gamma_x g_1 r \sqrt{c-r^2},\quad
    \bar{B}  =  -\frac{g_1\sqrt{c-r^2}}{r}\nonumber \\
    \bar{C}  & = & 1 + 2c_3\gamma_x r^2,\quad
    \bar{D}  =  -c_3\gamma_x  ( g_1^2(1+c-r^2) -2g_1 |g_1| \sqrt{c-r^2}    )     \nonumber \\
    \bar{A}_2 & = & 2c_3\gamma_x r |g_1| + 2c_3\gamma_x g_1 r \sqrt{c-r^2}, \quad
    \bar{B}_2  =  -\frac{g_1\sqrt{c-r^2}}{r}   \nonumber \\
    \bar{C}_2 & = & 1 + 2c_3\gamma_x r^2,\quad
    \bar{D}_2  =  -c_3\gamma_x  ( g_1^2(1+c-r^2) + 2g_1 |g_1| \sqrt{c-r^2}    ).
\end{eqnarray}
   \label{thm:thm1a}
\end{theorem}\vspace{-.17in}
\begin{proof}
Follows as an immediate consequence of Theorems 1 and 2 (and the ensuing discussions right after these theorems)  in \cite{Stojnicphretreal24}.
\end{proof}

As discussed in great detail in \cite{Stojnicphretreal24}, solving basic phase retrieval, ${\mathcal R}(A)$, using descending algorithms critically depends on analytical properties of $\xi(c,x)$. In particular, \emph{any} descending algorithm will converge to the global optimum and therefore solve phase retrieval if the so-called \emph{parametric manifold} -- ${\mathcal {PM}}(\alpha)$ -- associated with $\xi(c,x)$ has a \emph{single funneling point} (more on ${\mathcal {PM}}(\alpha)$ and funneling points definitions  as well as on particular ${\mathcal {PM}}(\alpha)$ shapes related to the problems of interest here can be found in Figures 2-4, 6, 8, and 9 and associated discussions in \cite{Stojnicphretreal24}; for earlier discussions regarding the importance of studying parametric structure in general see also a long series of RDT work \cite{StojnicCSetam09,StojnicICASSP10var,StojnicISIT2010binary,StojnicICASSP10block,StojnicRegRndDlt10}). Moreover,  \cite{Stojnicphretreal24} went further and uncovered that if one restricts the unknown vector to the unit $\ell_2$ ball  (through say a barrier method), then the portion of ${\mathcal {PM}}(\alpha)$ that critically impacts its shape is associated wit $c=1$ scenario. In Figure \ref{fig:fig2} we show the plain and lifted RDT estimates from Theorem \ref{thm:thm1a} for $c=1$ and $\alpha=1.4$ (the value $\alpha=1.4$ is chosen as a critical value for which the resulting curve is decreasing and with a single funneling point, i.e., it is the theoretical \emph{lifted} RDT estimate of the phase retrieval descending algorithms \emph{phase transition}). As can be seen, the lifted curve indeed decreases which ultimately leads to having $(c,x)=(1,1)$ as a desired single funneling point. At the same time, the curve remains fairly flat for a large portion of $[0,1]$ interval (the allowed set of overlap $x$ values). If one uses an overlap optimal spectral initializer, we from Figure \ref{fig:fig1} find that for $\alpha=1.4$ the starting overlap is no better than  $x^{(spec)}\approx0.6439$. Superimposing this value  (given in purple) on the lifted RDT curve in Figure \ref{fig:fig2}, one observes that if the best spectral initializer is used, the algorithm will still start in the so-called flat region. In theory that would not matter as the curve is of the decreasing type. However, as pointed out in \cite{Stojnicphretreal24}, in practice things are a bit different. Namely, theoretical values rely on strong concentrations which indeed happen in the assumed high-dimensional $n\rightarrow\infty$ regime. On the other hand, in practical algorithmic running the values for $n$ are finite and the concentrations will not always be perfect. This basically means that instead of being smooth the shown curve will in practice be slightly \emph{jittery}. Given that the curve is already fairly flat to begin with, such a slight jitteriness can easily cause local traps and ultimately prevent the algorithm from reaching the global optimum, i.e., from reaching the low right corner ($\bar{\phi}_0=1$ and $x=1$). One can summarize the above as: in theory ${\mathcal {PM}}(\alpha)$ is so to say both \emph{globally} and \emph{locally} smooth whereas in practice its a \emph{local jitteriness} is likely to emerge as well.

Given the above, one naturally may ask what would be the way to potentially mitigate the jitteriness effect and remedy the problem. Even without understanding the behavior of $\xi(c,x)$ (or even without noting that $\xi(c,x)$ plays any role in reaching the global optimum of ${\mathcal R}(A)$), one may follow the logic of ``safer compression'' and simply increase the number of measurements hoping that would create more favorable algorithmic decompression (recovery) conditions. In Figure \ref{fig:fig3} we show what actually happens if one indeed opts for safer compression and increases the number of measurements, i.e., increases the sample complexity ratio to, say, $\alpha=1.6$. We first observe from Figure \ref{fig:fig1} that the overlap of the optimal spectral initializer increases to $x^{(spec)}=0.7055$. This value is superimposed  in orange on the lifted RDT curve in Figure \ref{fig:fig3}. As figure shows, the lifted RDT curve is now visibly steeper and the superimposed initializer's overlap value is clearly outside the so-called flat region. This would suggest that the descending algorithms are indeed likely to be more resistant to jitteriness effect as one increases the sample complexity ratio.

 \begin{figure}[h]
\centering
\centerline{\includegraphics[width=1\linewidth]{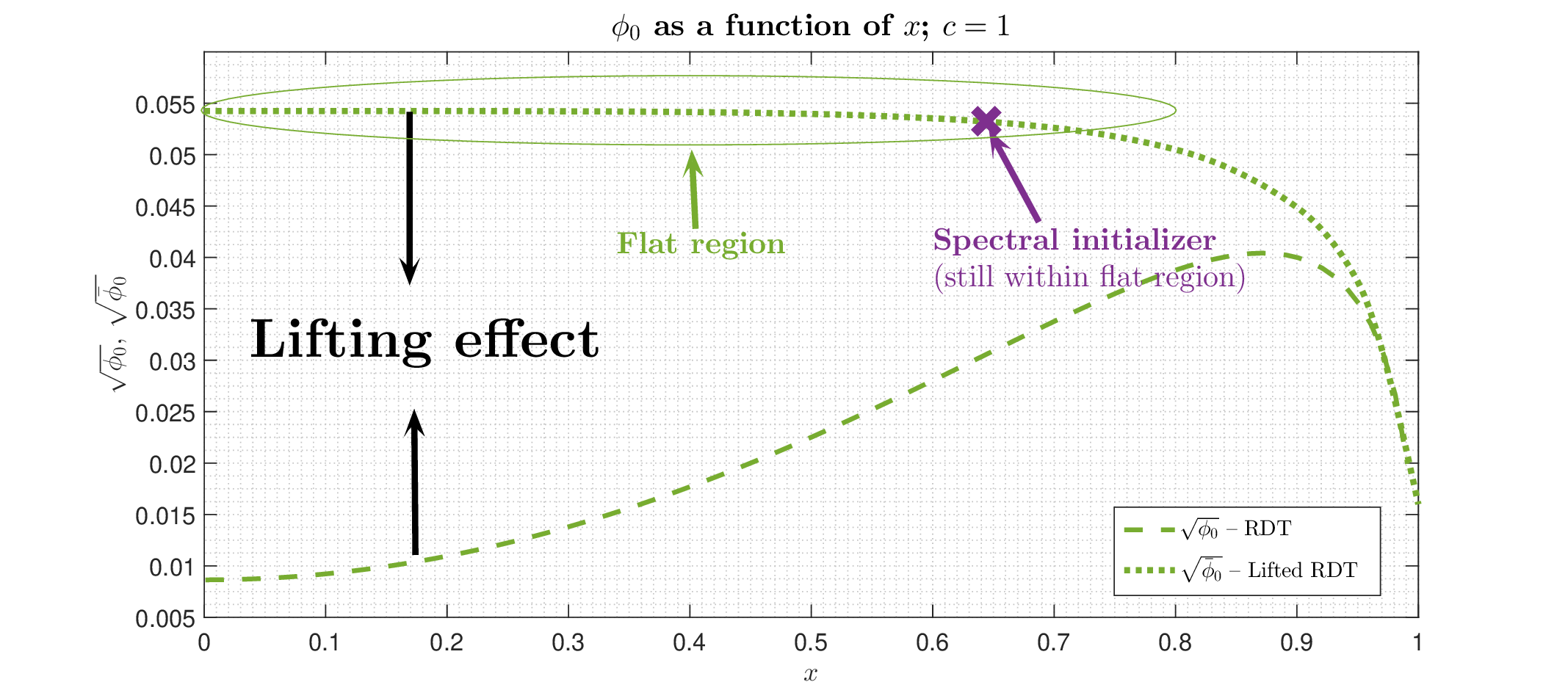}}
\caption{$\phi_0$ and $\bar{\phi}_0$ as functions of $x$; $c=1$ and $\alpha=1.4$}
\label{fig:fig2}
\end{figure}

\begin{figure}[h]
\centering
\centerline{\includegraphics[width=1\linewidth]{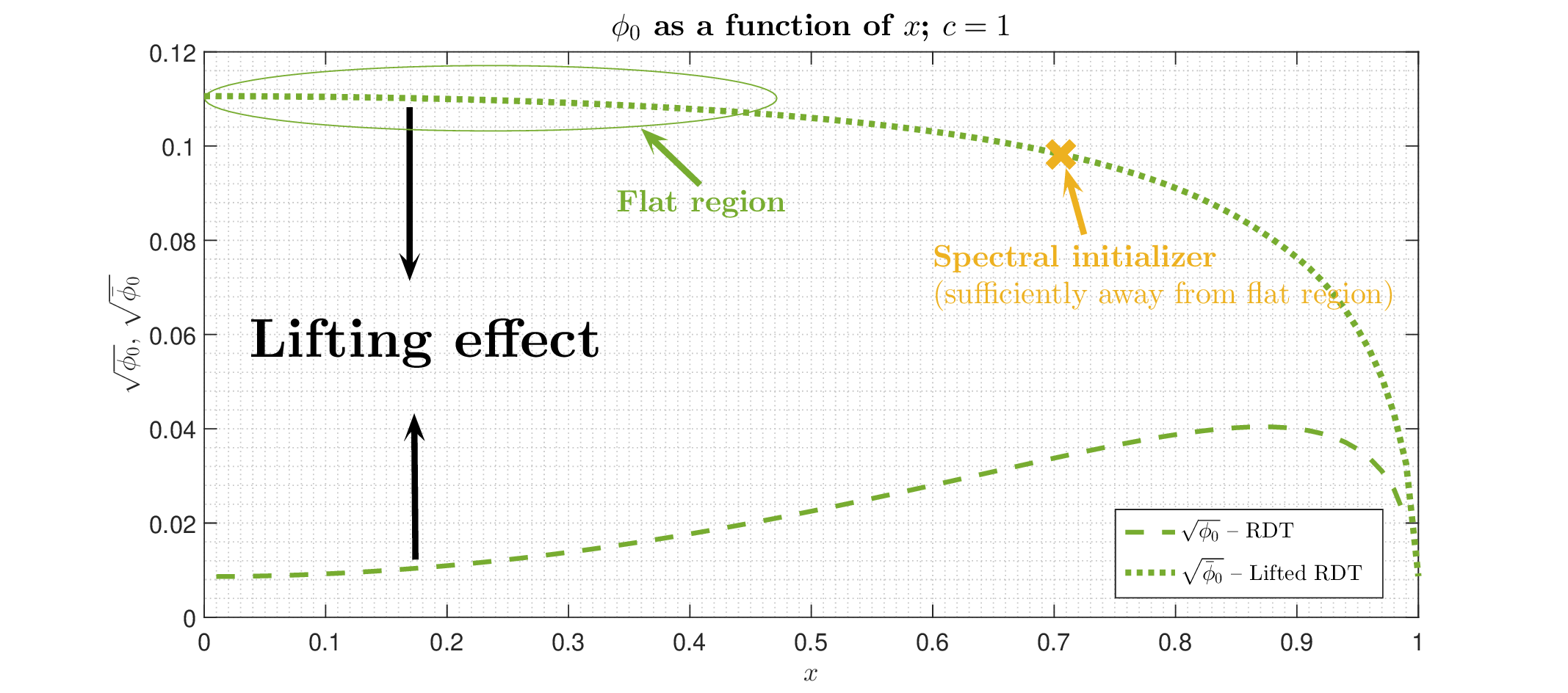}}
\caption{$\phi_0$ and $\bar{\phi}_0$ as functions of $x$; $c=1$ and $\alpha=1.6$}
\label{fig:fig3}
\end{figure}


While the flatness of the manifold and consequential jitteriness are the key factors impacting the practical reaching of the theoretical phase transitions, there are two other things that may have a prominent effect as well. First, the above assumes that the $\ell_2$ norm of the unknown $\x$ is limited while the algorithm is being run. That typically requires utilizing algorithmic techniques (like the above mentioned barrier method) which ensure that such a running is indeed in place. Such techniques, on the other hand, are more complex than say the standard plain gradient. If one is to use just a plain gradient then the above analysis needs to extend to $c>1$ scenarios. As \cite{Stojnicphretreal24}'s Section 4 demonstrates, the presented theory then implies that in such scenarios there is no guaranteed sample complexity ratio beyond which the descending algorithms generically converge to the global optimum. In practice though, things are a bit different. The datasets are typically such that the unfavorable scenarios  are unlikely to happen and simple plain gradient type of algorithmic schemes actually still work reasonably well (see Section 4.1 in \cite{Stojnicphretreal24} as well as the first next section below).

The second thing that one also needs to keep in mind is that the strong RDT is not in place and the estimated oversampling is likely to go down if one implements further lifting and ultimately \emph{fully lifted} RDT \cite{Stojnicflrdt23} (i.e., the results are likely to be even more favorable than what we presented in Figures \ref{fig:fig2} and \ref{fig:fig3}). None of this brings any conceptual changes to the above discussion, but may slightly change the concrete values of the minimal needed oversampling ratio. All in all, the following main conclusion can be taken as a rule of thumb:  in practical algorithmic running, one should slightly oversample (by say 10-20\%) the theoretically minimal sample complexity ratio. In other words, when facing concrete algorithmic implementations, a useful strategy might be to work in the regimes $10-20\%$ above the theoretical phase transition. We show in the next section numerical simulations that bring the empirical validity to such a strategy.

Before getting to the numerical results, we briefly recall on a discussion from \cite{Stojnicphretreal24} regarding potential other sources  that may be responsible for PR's generic algorithmic solvability. All of the above discussion is centered around the so called optimal objective landscape. As stated in \cite{Stojnicphretreal24}, it is likely that associated  parametric manifold can indeed be among the key reasons behind generic success of practically efficient algorithms. One however, has to keep in mind that many other random structures  intrinsic  features may contribute as well. For example,the entirety of the \cite{Stojnicphretreal24}'s discussion regarding \emph{overlap gap property} (OGP) \cite{Gamar21,GamarSud14,GamarSud17,GamarSud17a,AchlioptasCR11,HMMZ08,MMZ05,Montanari19} and \emph{local entropy} (LE) \cite{Bald15,Bald16,Bald20} applies here as well. Whether these or others related properties play any role with respect to PR's generic solvability remains to be seen.

\section{Numerical simulations}
 \label{sec:pract}

As mentioned above, in practical implementations one has to account for ${\mathcal {PM}}(\alpha)$'s local jitteriness and potential limitations/restrictions of the norm of the targeted vector, $\|\x\|_2$. The former will be addressed through oversampling. To emphasize the impact of the latter, we will run both, plain gradient and its a log barrier variant which allows to restrict $\|\x\|_2$.  We refer to \cite{Stojnicphretreal24} for complete set of technical details and here briefly summarize the main points related to algorithmic implementations.

The following barrier function is the objective of interest
\begin{eqnarray}
f_{bar}(t_0;\x) \triangleq  t_0\||A\bar{\x}|^2-|A\x|^2 \|_2^2 + \log\lp 1-\|\x\|_2^2 \rp
=t_0f_{plain}(\x) + \log\lp 1-\|\x\|_2^2 \rp,
\label{eq:practeq1}
\end{eqnarray}
where
\begin{eqnarray}
f_{plain}(\x) \triangleq \||A\bar{\x}|^2-|A\x|^2 \|_2^2.
\label{eq:practeq1a0}
\end{eqnarray}
Similarly to what was done in  \cite{Stojnicphretreal24}, we find it slightly more convenient to work with (derivative) smoother squared magnitudes  rather than just magnitudes. Also, the norm is taken to be smaller than 1 which in a way assumes a prior norm knowledge (as highlighted in \cite{Stojnicphretreal24}, such an assumption can be avoided by simply  rerunning $\sim 1/\epsilon$ times our procedures with different norm barriers; such a rerunning  would not change the algorithm's complexity order).

We test two optimization procedures, the plain gradient, $\mathbf{gradplain}$, and the following
 \begin{eqnarray}
\bl{\mathbf{hybrid:}} \qquad    \a^{(i+1)} &  =  & \mathbf{reshuffle} \lp \mathbf{gradplain} \lp \mathbf{reshuffle} \lp\mathbf{gradbar}\lp  \a^{(i)} \rp \rp \rp \rp, i=0,1,2,\dots,
\label{eq:practeq6}
\end{eqnarray}
alternation between plain gradient and its log barrier, $\mathbf{gradbar}$ (to help avoiding local traps, we in between also apply
a simple $\mathbf{reshuffle}$ procedure  which  takes a fraction (say $5-10\%$) of components a few times (rarely more than 10 times) and changes their signs). We rely on the overlap optimal spectral initializers (OptSpins) for $\a^{(0)}$, i.e. we take
\begin{equation}
\bl{\mbox{\emph{\textbf{overlap optimal spectral initialization:}}}} \qquad   \a^{(0)}=\x^{(spec)} \triangleq \mbox{max eigenvector} \lp A^T \mbox{diag} \hat{{\mathcal T}} \lp  | A\bar{\x} |^2  \rp   A\rp.
\label{eq:practeq6}
\end{equation}
The results obtained through both procedure $\mathbf{gradplain}$ and $\mathbf{hybrid}$ are shown in Figure \ref{fig:fig7}. To get a feeling as to what the effect of OptSpins on the simulated PR phase transition is, we also show in parallel the results obtained in \cite{Stojnicphretreal24} where a simple (overlap non-optimal) spectral initializer is used. In all tests $n=300$.
\begin{figure}[h]
\centering
\centerline{\includegraphics[width=1\linewidth]{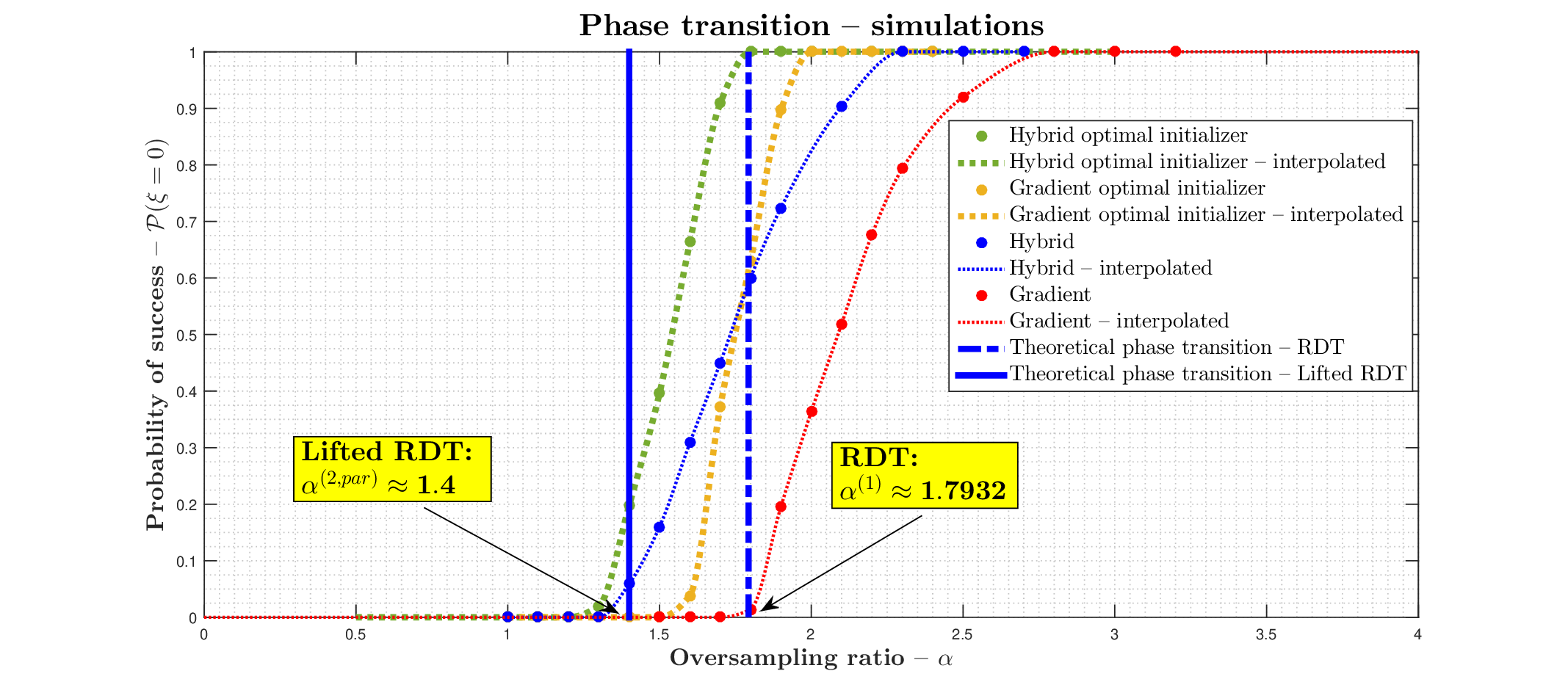}}
\caption{Simulated and theoretical RDT and lifted RDT phase transitions -- Hybrid vs Gradient}
\label{fig:fig7}
\end{figure}
As can be seen, a substantial improvement is observed for both procedures if the OptSpins are used. Moreover, both simulated phase transitions are a bit away from $\alpha=1.4$ (theoretical lifted RDT prediction), but are fairly close to its ``safer compression'' alternative, $\alpha=1.6$. This is true even though we have done simulations with squared magnitudes (opposed to the non-squared ones used in theoretical calculations) and with $f_{bar}$ objective (opposed to $f_{plain}$  (with non-squared magnitudes) in ${\mathcal R}(A)$ used in theoretical calculations).  As Figures 8, 9, and 14 (and their associated lengthy discussions) in \cite{Stojnicphretreal24} indicate, such deviations were indeed not expected to bring any substantial changes in algorithmic performance.

\subsection{Squared magnitudes}
 \label{sec:sqadj}

The $f_{plain}(\x)$  used in the above simulation is slightly different from the one used in  (\ref{eq:ex1a4}) and (\ref{eq:ex1a4a0}) and the ensuing analyses leading to Theorem \ref{thm:thm1a}. For the completeness, we below briefly revisit discussion from  \cite{Stojnicphretreal24} related to what happens if, instead of  (\ref{eq:ex1a4}) and (\ref{eq:ex1a4a0}), one conducts theoretical analyses utilizing
\begin{eqnarray}
 {\mathcal R}^{(sq)}A): \qquad \qquad    \min_{\x,\z} & & \||A\bar{\x}|^2-|\z|^2 \|_2^2\nonumber \\
  \mbox{subject to} & &  A\x=\z, \label{eq:sqadjex1a4}
\end{eqnarray}
and corresponding
 \begin{eqnarray}
\hspace{-.8in}\bl{\textbf{\emph{f-pro -- squared magnitudes:}}} \qquad\qquad  \xi^{(sq)}(c,x) \triangleq \min_{\x,\z} & & \||A\bar{\x}|^2-|\z|^2 \|_2^2\nonumber \\
  \mbox{subject to} & &  A\x=\z \nonumber \\
  & & \x^T\bar{\x}=x \nonumber \\
  & & \|\x\|_2^2=c. \label{eq:sqadjex1a4a0}
\end{eqnarray}

\begin{theorem} Assume the setup of Theorem \ref{thm:thm1a} and let $\xi^{(sq)}(c,x)$ be as in (\ref{eq:sqadjex1a4a0}). Then
\begin{eqnarray}
  \lim_{n\rightarrow\infty}\mP_{ A } \lp \frac{\xi^{(sq)}(c,x)}{n}  \geq  \bar{\phi}^{(sq)}_0  \geq   \phi^{(sq)}_0 \rp \longrightarrow 1,\label{eq:ta17a0}
\end{eqnarray}
where $\phi^{(sq)}_0$ and $\bar{\phi}^{(sq)}_0$ are squared magnitudes objective plain and lifted RDT estimates. In particular, for $\phi^{(sq)}_0$ we have
\begin{eqnarray}
 \phi^{(sq)}_0
  & = &
  \max_{r_y>0}
  \lp
  \alpha
  f^{(sq)}_q
  -  r^2 r_y \rp,
\label{eq:thm2asqadjhrd7a4}
 \end{eqnarray}
where $\g_i^{(0)}$ and $\g_i^{(1)}$ are iid standard normals, $\cG=\{\g_i^{(0)},\g_i^{(1)}\}$, and
\begin{eqnarray}
f_q^{(sq)} = \mE_{\cG} \min_{|\z_i|\in{\mathcal Z}^{(sq)}}\lp \||\g_i^{(0)}|^2-|\z_i|^2 \|_2^2
  +\||\g_i^{(0)}x   +  \g_i^{(1)}r| -|\z_i|\|_2^2 r_y \rp,
\label{eq:thm2ahrd7a2a0}
 \end{eqnarray}
with
\begin{eqnarray}
p_c & = & \frac{1}{2}\lp r_y -2|\g^{(0)}| \rp ,\quad q_c  =   -\frac{r_y}{2} |\g_i^{(0)}x   +  \g_i^{(1)}r |  \nonumber \\
      a_{c,1} &  =  & \lp -\frac{q_c}{2} + \sqrt{ \frac{q_c^2}{4} + \frac{p_c^3}{27} }  \rp^{\frac{1}{3}}
 - \frac{p_c}{3     \lp -\frac{q_c}{2} + \sqrt{ \frac{q_c^2}{4} + \frac{p_c^3}{27} }  \rp^{\frac{1}{3}}  }
\nonumber \\
     a_{c,2} &  =  &  a_{c,1}  \frac{-1-\sqrt{-3}}{2}, \quad
      a_{c,3}   =   a_{c,1}  \frac{-1+\sqrt{-3}}{2}\nonumber \\
 {\mathcal Z}^{(sq)} & = &
\begin{cases}
  \{a_{c,1},0 \}, & \mbox{if $\frac{q_c^2}{4} + \frac{p_c^3}{27} \geq 0$}  \\
  \{a_{c,1},a_{c,2},a_{c,3},0 \} , & \mbox{otherwise}.
\end{cases}.
\label{eq:thm2asqadjhrd7a1}
 \end{eqnarray}
 Moreover, for $\bar{\phi}^{(sq)}_0$ we have
\begin{eqnarray}
 \bar{\phi}^{(sq)}_0
&  = &
 \max_{c_3> 0}  \max_{r_y>0}\min_{\gamma>0}
\Bigg .\Bigg(
 \frac{c_3}{2} r^2r_y^2 + \gamma
 -\frac{\alpha}{c_3} \log \lp f_{q}^{(sq,lift)}\rp
 - \hat{\gamma}_{sph}  +\frac{1}{2c_3} \log \lp  1  -  \frac{c_3rr_y  }{2\hat{\gamma}_{sph}}     \rp
\Bigg.\Bigg),
\label{eq:thm2asqadjplhrd16a0}
 \end{eqnarray}
 where
\begin{eqnarray}
f_{q}^{(sq,lift)}
 & = &
\int_{-\infty}^{\infty}
\int_{-\infty}^{\infty}
 \frac{e^{
 -c_3 \lp \min_{|\z_i|\in{\bar{\mathcal Z}}_i^{(sq)}}\lp \||\g_i^{(0)}|^2-|\z_i|^2 \|_2^2
  +\||\g_i^{(0)}x   +  \g_i^{(1)}r| -|\z_i|\|_2^2 \bar{r}_y \rp \rp
 -\frac{ \lp \g_i^{(0)} \rp^2 + \lp \g_i^{(1)} \rp^2     } {2}  }}{\sqrt{2\pi}^2} d\g_i^{(0)} d\g_i^{(1)}, \nonumber \\
  \label{eq:thm2asqadjplhrd11}
 \end{eqnarray}
with
\begin{eqnarray}
\bar{r}_y & =  & \frac{r_y^2}{4\gamma},\quad \bar{p}_c  =  \frac{1}{2}\lp \bar{r}_y -2|\g^{(0)}| \rp, \quad
\bar{q}_c  =   -\frac{\bar{r}_y}{2} |\g_i^{(0)}x   +  \g_i^{(1)}r |.
\nonumber \\
     \bar{a}_{c,1} &  =  & \lp -\frac{\bar{q}_c}{2} + \sqrt{ \frac{\bar{q}_c^2}{4} + \frac{\bar{p}_c^3}{27} }  \rp^{\frac{1}{3}}
 - \frac{\bar{p}_c}{3     \lp -\frac{\bar{q}_c}{2} + \sqrt{ \frac{\bar{q}_c^2}{4} + \frac{\bar{p}_c^3}{27} }  \rp^{\frac{1}{3}}  } \nonumber \\
     \bar{a}_{c,2}  & = &   \bar{a}_{c,1}  \frac{-1-\sqrt{-3}}{2},\quad
      \bar{a}_{c,3}   =    \bar{a}_{c,1}  \frac{-1+\sqrt{-3}}{2}
      \nonumber \\
\bar{{\mathcal Z}}_i^{(sq)} &= &
\begin{cases}
  \{\bar{a}_{c,1},0 \}, & \mbox{if $\frac{\bar{q}_c^2}{4} + \frac{\bar{p}_c^3}{27} \geq 0$}  \\
  \{\bar{a}_{c,1},\bar{a}_{c,2},\bar{a}_{c,3},0 \} , & \mbox{otherwise}.
\end{cases},
 \label{eqsqadj:thm2aplhrd4}
\end{eqnarray}
\label{thm:thm2a}
\end{theorem}\vspace{-.17in}
\begin{proof}
Follows immediately as a direct consequence of Theorems 3 and 4 and the associated ensuing discussions right after these theorems in \cite{Stojnicphretreal24}.
\end{proof}

Numerical results obtained through the above theorem are shown in Figure \ref{fig:fig4}. As discussed in great detail in \cite{Stojnicphretreal24}, precision of the underlying numerical evaluations is not necessarily on the level of the corresponding ones from the previous sections that relate to the non-squared magnitudes objectives. Since the values of $\sqrt{\bar{\phi}^{(sq)}_0}$ often differ on the fifth decimal for a large portion of the allowed range of $x$, it is even difficult to say that the lifted curve definitively has  decreasing behavior and that $\alpha=1.4$ is indeed the phase transition. On the other hand, as we have seen in earlier section, for practical functioning of descending algorithms that may not be overly relevant. Instead, one finds as of primary concern whether or not the starting point is outside the so-called \emph{flat region}. As Figure \ref{fig:fig4} shows, for $\alpha=1.4$ the spectral initializer is well within the flat region. That is however to be expected if one does not believe that there should be much of a difference between the non-squared and squared magnitudes objectives. Namely, in the corresponding non-square magnitudes scenario shown in Figure  \ref{fig:fig2} one also observes that the optimal spectral initializer falls well within the flat region. Naturally a  similar non-squared/squared lack of difference  trend is expected to continue as the oversampling is increased to say $\alpha=1.6$. Figure \ref{fig:fig4a} shows that this is indeed true (simulations results from Figure \ref{fig:fig7} provide an empirical confirmation as well). In other words, the main takeaway is that there is really not that much conceptual difference between the non-squared and squared magnitudes scenarios -- the former is more convenient for theoretical analysis whereas the latter is a bit more amenable for practical algorithmic implementations.

\begin{figure}[h]
\centering
\centerline{\includegraphics[width=1\linewidth]{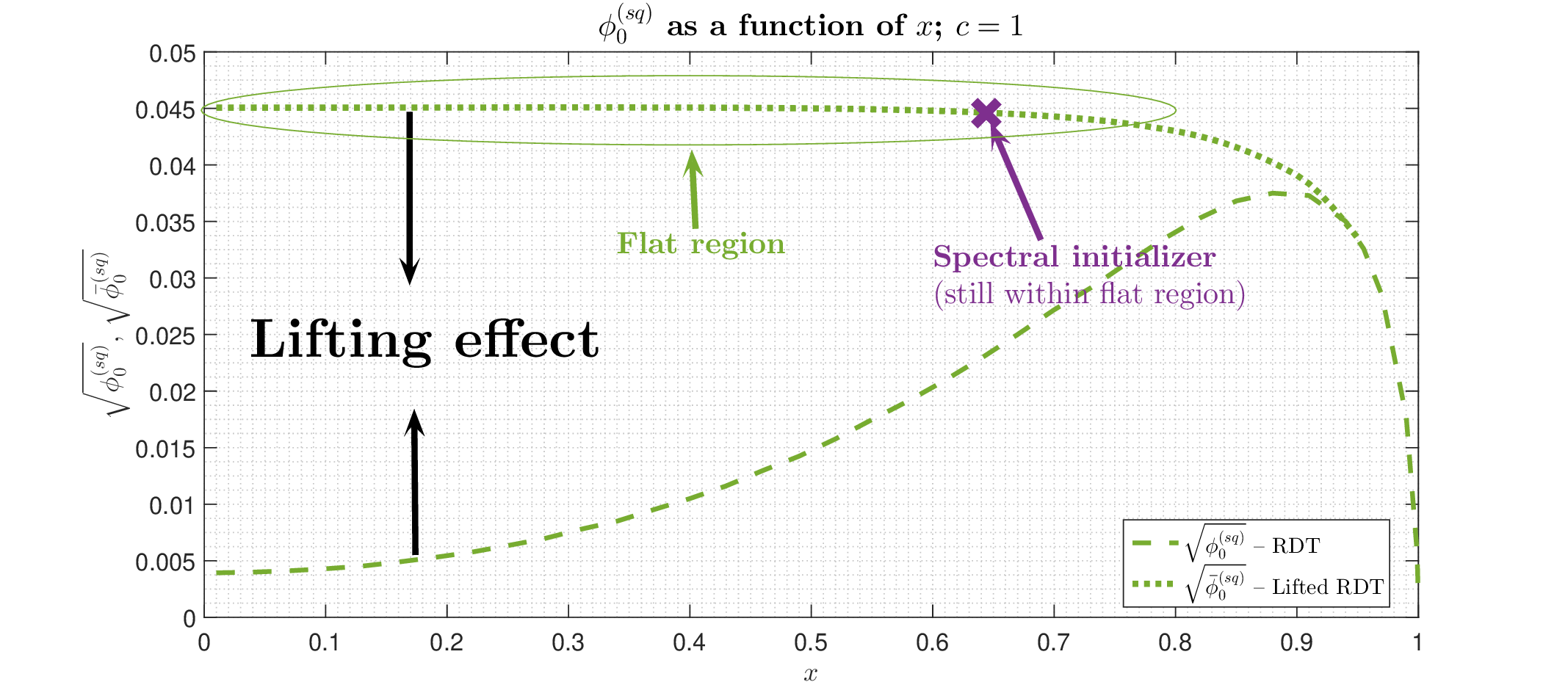}}
\caption{$\phi^{(sq)}_0$ and $\bar{\phi}^{(sq)}_0$ as functions of $x$; $c=1$ and $\alpha=1.4$}
\label{fig:fig4}
\end{figure}

\begin{figure}[h]
\centering
\centerline{\includegraphics[width=1\linewidth]{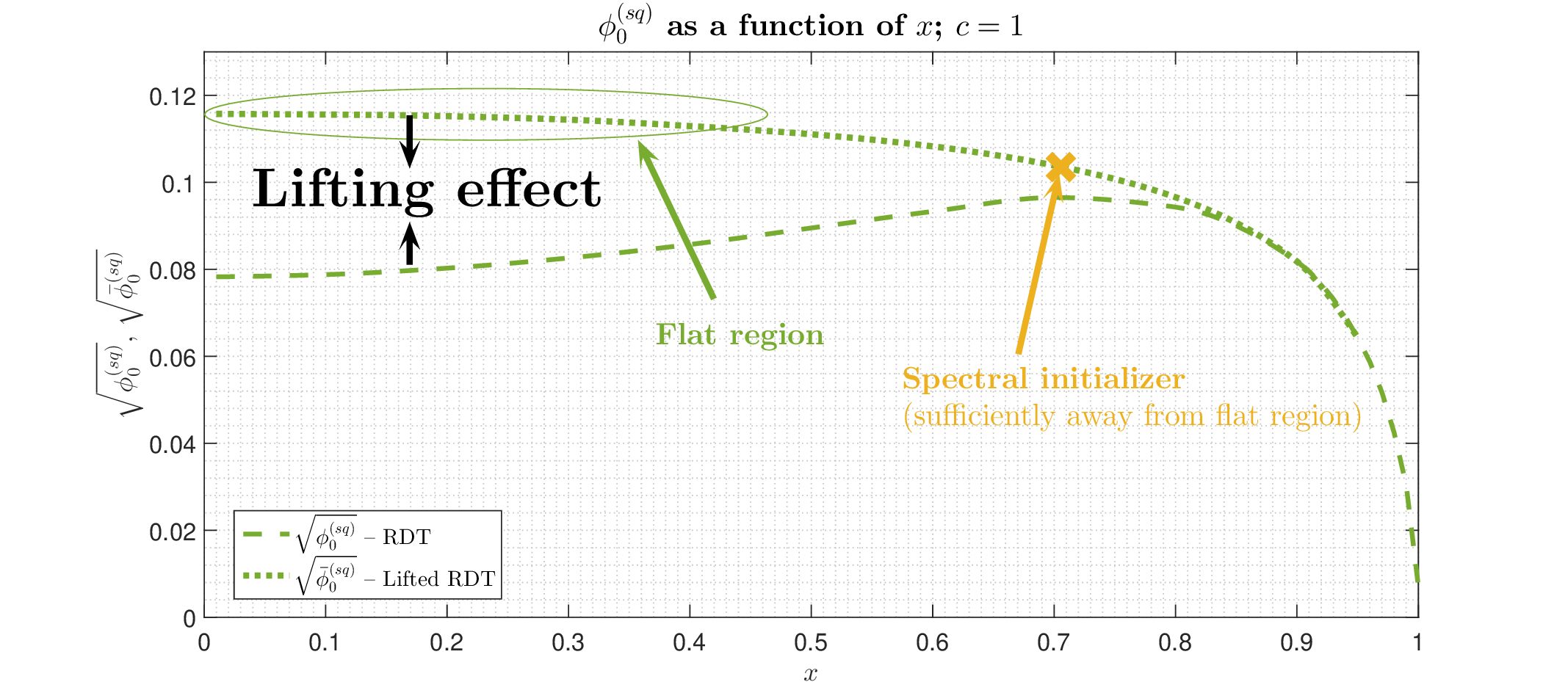}}
\caption{$\phi^{(sq)}_0$ and $\bar{\phi}^{(sq)}_0$ as functions of $x$; $c=1$ and $\alpha=1.6$}
\label{fig:fig4a}
\end{figure}


\section{Lowering dPR phase transitions via optimal initializers}
\label{sec:thimpact}

In the previous sections we mostly focused on the impact that optimal spectral initializers have regarding (practically) achieving the dPR's theoretical phase transitions. It is not that difficult to see that they can be used as state of the art concept to conceptually even lower the dPR's theoretical phase transitions as well. Namely, Theorems \ref{thm:thm1a} and \ref{thm:thm2a} consider the so-called generic convergence to the global optimum and consequently insist that underlying ${\mathcal {PM}}$s have a single funneling point. If one steps away for a moment from such a generic consideration and instead considers global convergence on a case by case basis then  the  single funneling point ${\mathcal {PM}}$s are only a sufficient condition for dPR's convergence to the global optimum. On the other hand, a necessary condition, for example, would  assume that one starts on a portion of ${\mathcal {PM}}$ from which all descending paths are converging to the global optimum. Keeping this in mind one can look at Figure \ref{fig:fig2} and reevaluate the corresponding plain and lifted RDT curves for various $\alpha<1.4$ and then determine the critical one for which the corresponding spectral initializers fall within the zone of convergence towards the global optimum. For $c=1$ this means that the spectral initializer is on the right descending side of the lifted RDT curve (for $0\leq c< 1$, on the other hand, one actually has to check that the spectral initializer's location on ${\mathcal {PM}}$ is such that from there it must funnel down to the global optimum). In generic scenarios covered by Theorems \ref{thm:thm1a} and \ref{thm:thm2a}, $c=1$ turns out to be the critical. That might also be the case for  $\alpha<1.4$ but it might not and one has to be additionally careful. We skip doing such evaluations as, for practical dPRs and initializers that we consider, their relevance seems  unsubstantiated. However, for different dPR's and/or initializers they might have an important role and that is the reason why we mention it as a useful state of the art concept that is fully covered by the theory that we developed and presented above.

\section{Conclusion}
\label{sec:conc}

We considered the impact that spectral initializers have on the theoretical limits of \emph{descending} phase retrieval algorithms (dPR). \cite{Stojnicphretreal24} studied generic dPRs (which include as special cases any forms of gradient descent or Wirtinger flow) and provided a \emph{precise} theoretical analysis of their performance. Since dPR need starting points (initializers), the results from \cite{Stojnicphretreal24} are here complemented with the corresponding ones related to the initializers. In particular, \cite{Stojnicphretreal24} established that dPRs ability to solve phase retrieval (PR) critically depends on the structure of the associated so-called parametric manifold, ${\mathcal {PM}}$, and additionally uncovered  that the overlap between the algorithmic solution and the true signal is a key ${\mathcal {PM}}$'s component. We here consider the so-called \emph{overlap optimal} spectral initializers (OptSpins) as dPR's starting points and  develop a generic \emph{Random duality theory} (RDT) based program to statistically characterize them. Functional structure of OptSpins is determined and the starting overlaps that they provide for the dPRs are \emph{precisely} evaluated. A precise characterization of the starting overlap allows to determine if ${\mathcal {PM}}$'s so-called \emph{flat regions} can be successfully circumvented (\emph{flat regions} are highly susceptible to \emph{local jitteriness} which represents a key obstacles for dPR to reach global optimum and solve PR).

The presented theoretical analysis allows to observe two key points: \textbf{\emph{(i)}} dPR's  theoretical phase transition (critical $\alpha$ above which they solve PR) might be difficult to practically achieve as flat regions are large and OptSpins fail to avoid them; and \textbf{\emph{(ii)}} Opting for so-called ``\emph{safer compression}'' and slightly increasing  $\alpha$ (by say $15\%$) shrinks flat regions and allows OptSpins to fall outside them which ultimately enables dPRs to solve PR.

A sizeable set of numerical experiments is conducted as well so that the theoretical results are brought to practical relevance in a clearer way. We implemented in parallel plain gradient descent and a hybrid combination of plain and log barrier gradient descent algorithms. Even though the simulations were run for fairly small dimensions ($n=300$) (which implies strong parametric manifold local jitteriness), the  simulated and theoretical phase transitions are fairly close to each other. In particular, a \emph{safer compression} theoretical phase transition adjustment is introduced and shown to almost identically match the presented simulated results.


Developed theoretical methods are completely different from the typical spectral ones that rely on random matrix and free probability theory. Our methods rely on \emph{Random duality theory} (RDT), are much simpler, produce results that match the ones of the spectral methods, and are more general and adaptable to various signal structuring scenarios. Their generic character allows for  many further extensions and generalizations. For example, extensions to incorporate stability and robustness analyses and adaptations to account for different measurement strategies and complex domain are only a few drops in an ocean of endless possibilities. While conceptual ingredients needed for any of these considerations are closely related to what is presented here and in \cite{Stojnicphretreal24}, accompanying technical details are problem specific and we present them in separate companion papers.

\begin{singlespace}
\bibliographystyle{plain}
\bibliography{nflgscompyxRefs}
\end{singlespace}


\end{document}